\documentclass[twoside]{article}

\usepackage[accepted]{aistats2024}

\usepackage[round]{natbib}

\usepackage[colorinlistoftodos]{todonotes}

\usepackage{mathtools}
\usepackage{amsfonts}
\usepackage{thm-restate}
\usepackage{algorithm}
\usepackage{algorithmic}
\usepackage{amsthm}
\usepackage{amsmath}
\usepackage{bm}
\usepackage{bbm}

\newtheorem{proposition}{Proposition}

\newtheorem{theorem}{Theorem}
\newtheorem{lemma}{Lemma}

\newcommand{\argmin}{\operatornamewithlimits{argmin}}

\theoremstyle{remark}

\newtheorem{example}[theorem]{Example}

\begin{document}

%

%

\twocolumn[

\aistatstitle{Inverse Reinforcement Learning with Sub-optimal Experts}

\aistatsauthor{ Riccardo Poiani \And Gabriele Curti \And  Alberto Maria Metelli \And Marcello Restelli}

\aistatsaddress{ Politecnico di Milano \And  Politecnico di Milano \And Politecnico di Milano \And Politecnico di Milano } ]

\begin{abstract}
Inverse Reinforcement Learning (IRL) techniques deal with the problem of deducing a reward function that explains the behavior of an expert agent who is assumed to act optimally in an underlying unknown task. In several problems of interest, however, it is possible to observe the behavior of multiple experts with different degree of optimality (e.g., racing drivers whose skills ranges from amateurs to professionals). For this reason, in this work, we extend the IRL formulation to problems where, in addition to demonstrations from the optimal agent, we can observe the behavior of multiple sub-optimal experts. Given this problem, we first study the theoretical properties of the class of reward functions that are compatible with a given set of experts, i.e., the \emph{feasible reward set}. Our results show that the presence of multiple sub-optimal experts can significantly shrink the set of compatible rewards. 
Furthermore, we study the statistical complexity of estimating the feasible reward set with a generative model. To this end, we analyze a uniform sampling algorithm that results in being minimax optimal whenever the sub-optimal experts' performance level is sufficiently close to the one of the optimal agent.
\end{abstract}

\section{INTRODUCTION}
\emph{Inverse Reinforcement Learning} \citep[IRL,][]{ng2000algorithms} deals with the problem of recovering a reward function that explains the behavior of an expert agent who is assumed to act optimally in an underlying unknown task. Over the years, the IRL problem has consistently captured the attention of the research community (see, for instance, \citet{arora2021survey} and \citet{adams2022survey} for in-depth surveys). Indeed, this general scenario, where the reward function needs to be learned, emerges in numerous real-world applications. A prime example of this arises from human-in-the-loop settings \citep{mosqueira2023human}, where the expert is a human solving a task, and an explicit specification of the human's goal in the form of a reward function is often unavailable. Notably, humans encounter difficulty in expressing their intentions in the form of an underlying reward signal, preferring instead to demonstrate what they perceive as the correct behavior. Once we retrieve a reward function, (i) we obtain explicit information for understanding the expert's choices, and, furthermore, (ii) we can utilize it to train reinforcement learning agents, even under shifts in the features of the underlying system.

Since the seminal work of \citet{ng2000algorithms}, IRL has emerged as a significantly complex task. Indeed, one of its primary challenges lies in the intrinsic \emph{ill-posed} nature of the problem, as multiple reward functions that are compatible with the expert's behavior exist. Recently, a promising avenue of research \citep{metelli2021provably,lindner2022active,metelli2023towards} has tackled this ambiguity issue from an intriguing perspective. Specifically, this strand of works focuses on estimating \emph{all} the reward functions that are compatible with the observed demonstration, thereby postponing the selection of the reward function and directing their focus solely on the expert's intentions.

Nevertheless, these approaches fall short in modeling more complex situations that arise in the real world. Indeed, in several problems of interest, it is possible to observe the behavior of multiple agents with different degrees of expertise. As an illustrative example, we can consider the human-in-the-loop settings mentioned above. Imagine, indeed, that we are interested in recovering reward functions that explain the intent behind racing drivers. In this scenario, racing car companies typically have access to a variety of drivers with diverse skills, including professionals, test drivers, and emerging talents from developmental programs. In this context, while the focus is typically on the reward function of professional drivers, we expect a proficient IRL method to effectively leverage demonstrations and information provided by drivers with lower expertise. Indeed, from an intuitive perspective, if we have information on the degree of expertise of other drivers, we can expect that, by exploiting their demonstrations, we can reduce the inherent ambiguity of IRL problems. For this reason, in this work, we extend the IRL formulation to settings where, in addition to demonstrations from an optimal agent, we can observe the behavior of multiple sub-optimal experts, of which we know an index of their sub-optimality. 


More specifically, we will be primarily focused in answering the following theoretical questions:
\begin{itemize}
\item[(Q1)] How does the presence of sub-optimal experts affects the class of reward functions that are compatible with the observed behavior? Can they limit the intrinsic ambiguity that affects IRL problems?
\item[(Q2)] What is the statistical complexity of estimating the set of reward functions that are compatible with a given set of experts? How does it compare against the one of single-experts IRL problems?
\end{itemize}

\paragraph{Contributions and Outline} After providing the necessary notation and background, we introduce the novel problem of Inverse Reinforcement Learning with multiple and sub-optimal experts (Section \ref{sec:prel}). We then proceed by studying the \emph{theoretical properties} of the class of reward functions that are compatible with a given set of experts under the assumption that an upper bound on the performance between a sub-optimal agent and the optimal expert is available to the designer of the IRL system (Section \ref{sec:reward-set}). More precisely, our findings indicate that having multiple sub-optimal experts can significantly shrink the set of compatible rewards, thereby \emph{limiting} the ambiguity issue that affects the IRL problem. Leveraging our previous results, we continue by studying the \emph{statistical complexity} of estimating the feasible reward set with a generative model (Section \ref{sec:learning}). To this end, after formally introducing a Probabilistic Approximately Correct \citep[PAC,][]{even2002pac} framework, we derive a novel lower bound on the number of samples that are required to obtain an accurate estimate of the feasible reward set. Then, we present a uniform sampling algorithm and analyze its theoretical guarantees. Our results show that (i) the IRL problem with sub-optimal experts is statistically harder than the single agent IRL setting, and (ii) that the uniform sampling algorithm is minimax optimal whenever the sub-optimal experts’ performance level is sufficiently close to the one of the optimal agent. Finally, we conclude with a discussion on existing works (Section \ref{sec:related-works}) and by highlighting potential avenues for future research (Section \ref{sec:conclusions}).

%

\section{PRELIMINARIES}\label{sec:prel}
In this section, we provide the notation and essential concepts employed throughout this document.

\paragraph{Notation}
Consider a finite set $\mathcal{X}$, we denote with $\Delta^{\mathcal{X}}$ the set of probability measures over $\mathcal{X}$. Let $\mathcal{Y}$ be a set, we denote with $\Delta^{\mathcal{X}}_{\mathcal{Y}}$ the set of functions $f: \mathcal{Y} \rightarrow \Delta^{\mathcal{X}}$. Given $f \in \mathbb{R}^n$, we denote with $\|f\|_\infty$ the infinite norm of $f$. Let $\mathcal{X}$ and $\mathcal{X}'$ be two non-empty subsets of a metric space $\left(\mathcal{Y}, d \right)$, we define the Hausdorff distance \citep{rockafellar2009variational} between $\mathcal{X}$ and $\mathcal{X}'$ as:
\begin{equation*}
\resizebox{\linewidth}{!}{$\displaystyle
H_d(\mathcal{X},\mathcal{X}') =\max\left\{ \sup_{x \in \mathcal{X}} \inf_{x' \in \mathcal{X}'} d(x, x'),  \sup_{x' \in \mathcal{X}'} \inf_{x \in \mathcal{X}} d(x, x')  \right\}.$
}%
\end{equation*}
Notice that the Hausdorff distance is directly dependent on the metric $d$. Finally, given an integer $x \in \mathbb{N}_{>0}$, we denote with $\mathbf{1}_{x}$ the $x$-dimensional vector given by $\left(1, \dots, 1\right)^\top$.

\paragraph{Markov Decision Processes}
A Markov Decision Process \emph{without a reward function} (MDP\textbackslash R) is defined as a tuple $\mathcal{M}=\left(\mathcal{S}, \mathcal{A}, p, \gamma \right)$, where $\mathcal{S}$ is the set of states, $\mathcal{A}$ is the set of actions, $p \in \Delta_{\mathcal{S} \times \mathcal{A}}^{\mathcal{S}}$ denotes the transition probability kernel, and $\gamma \in [0,1)$ is the discount factor. In this paper, we consider finite state and action spaces, namely $|\mathcal{S}| = S$ and $|\mathcal{A}| = A$. A Markov Decision Process \citep[MDP,][]{puterman2014markov} is obtained by combining an MDP\textbackslash R $\mathcal{M}$ with a reward function $r \in \mathbb{R}^{\mathcal{S} \times \mathcal{A}}$. Without loss of generality, we assume reward functions bounded in $[0,1]$. We denote with $\mathcal{M} \cup r$ the resulting MDP. The behavior of an agent is described by a policy $\pi \in \Delta_{\mathcal{S}}^{\mathcal{A}}$, that, for each state, prescribes a probability distribution over actions. 

\paragraph{Operators} Consider $f \in \mathbb{R}^{\mathcal{S}}$ and $g \in \mathbb{R}^{\mathcal{S} \times \mathcal{A}}$. We denote with $P$ and $\pi$ the operators that are induced by the transition model $p$ and the policy $\pi$ respectively. More specifically, $Pf(s,a) = \sum_{s' \in \mathcal{S}} p(s'|s,a)f(s')$, and $\pi g(s) = \sum_{a \in \mathcal{A}} \pi(a|s)g(s,a)$. Moreover, we introduce the operators $E$ and $\bar{B}^{\pi}$ defined in the following way: $Ef(s,a) = f(s)$ and $\left(\bar{B}^{\pi}g\right)(s,a) = \mathbbm{1}\left\{ \pi(a|s)=0 \right\} g(s,a)$. Finally, we define $d^{\pi} f$ as the expectation of $f$ under the discounted occupancy measure. More formally $d^\pi f = \left( I_{\mathcal{S}} - \gamma \pi P \right)^{-1} f = \sum_{t=0}^{+\infty} (\gamma \pi P)^t f$.

\paragraph{Value Functions and Optimality}
Given an MDP $\mathcal{M} \cup r$ and a policy $\pi$, the \emph{Q-function} $Q^{\pi}_{\mathcal{M} \cup r}(\cdot)$ represents the expected discounted sum of rewards collected in $\mathcal{M} \cup r$ starting from $(s,a)$ and following policy $\pi$. More formally:
\begin{align*}
Q^{\pi}_{\mathcal{M} \cup r}(s,a) = \mathbb{E}\left[ \sum_{t=0}^{+\infty} \gamma^t r(s_t,a_t) | s_0=s, a_0=a \right],
\end{align*}
where the expectation is taken w.r.t. the stochasticity of the policy and the environment, that is $s_{t+1} \sim p(\cdot|s_t,a_t)$ and $a_t \sim \pi(\cdot | s_t)$. Similarly, the \emph{V-function} $V^{\pi}_{\mathcal{M} \cup r}$ represents the expectation of the $Q$-function over the action space, namely $V^{\pi}_{\mathcal{M} \cup r} = \pi Q^{\pi}_{\mathcal{M} \cup r}$. The \emph{advantage function} $A^{\pi}_{\mathcal{M} \cup r} = Q^{\pi}_{\mathcal{M} \cup r} - EV^{\pi}_{\mathcal{M} \cup r}$ represents the immediate gain of taking a given action, rather than following policy $\pi$. A policy $\pi^*$ is optimal if it has non-positive advantage in each-state action pair; namely $A^{\pi^*}_{\mathcal{M} \cup r} \le 0$ holds element-wise.

\paragraph{Inverse Reinforcement Learning}
An Inverse Reinforcement Learning \citep[IRL,][]{ng2000algorithms} problem is defined as a tuple $\mathfrak{B} = \left(\mathcal{M}, \pi_E \right)$, where $\mathcal{M}$ is an MDP\textbackslash R and $\pi_E \in \Delta^{\mathcal{A}}_{\mathcal{S}}$ is an expert policy. Given a reward function $r \in \mathbb{R}^{\mathcal{S} \times \mathcal{A}}$, we say that $r$ is \emph{feasible} for $\mathfrak{B}$ if it is compatible with the behavior of the expert, namely $\pi_E$ is an optimal policy for the MDP $\mathcal{M} \cup r$. We denote with $\mathcal{R}_{\mathfrak{B}}$ the set of feasible reward functions, namely:
\begin{align}\label{eq:feasible-set-irl}
\mathcal{R}_{\mathfrak{B}} = \left\{ r \in [0,1]^{\mathcal{S} \times \mathcal{A}}: A^{\pi_E}_{\mathcal{M} \cup r} \le 0  \right\}.
\end{align}
The set $\mathcal{R}_{\mathfrak{B}}$ takes the name of \emph{feasible reward set} \citep{metelli2021provably,lindner2022active,metelli2023towards}. To characterize the set $\mathcal{R}_{\mathfrak{B}}$, \citet{metelli2021provably} have shown that a reward function $r$ belongs to $\mathcal{R}_{\mathcal{B}}$ if and only if there exists $\zeta \in \mathbb{R}^{\mathcal{S} \times \mathcal{A}}_{\ge 0}$ and $V \in \mathbb{R}^{\mathcal{S}}$ such that:
\begin{align}\label{eq:reward-char}
r = -\bar{B}^{\pi_E} \zeta + (E-\gamma P)V.
\end{align}
In other words, each reward function in $\mathcal{R}_{\mathfrak{B}}$, is expressed as a sum of two components. The first one, $-\bar{B}^{\pi_E} \zeta$, which is non-zero only when $\pi_{E}(a|s) = 0$, can be interpreted as the advantage function $A^{\pi_E}_{\mathcal{M} \cup r}$. The second one, $(E-\gamma P)V$, instead, can be interpreted as a reward-shaping via function $V$, which is widely recognized to maintain the optimality of the expert's policy \citep{ng2000algorithms}. Given this interpretation, it follows that $\|V\|_\infty \le (1-\gamma)^{-1}$ and  $\|\zeta\|_\infty \le (1-\gamma)^{-1}$.

\paragraph{IRL with Sub-optimal Experts}
In this work, we extend the IRL formulation to problems where, in addition to demonstrations from an optimal expert, we can observe the behaviors of multiple and sub-optimal agents. More precisely, we define an Inverse Reinforcement Learning problem with multiple and Sub-optimal Experts (IRL-SE) as a tuple $\bar{\mathfrak{B}} = \left( \mathcal{M}, \pi_{E_1}, \left( \pi_{E_i} \right)_{i=2}^{n+1}, \left( \xi_i \right)_{i=2}^{n+1} \right)$, where $\mathcal{M}$ is an MDP\textbackslash R, $\pi_{E_1}$ is the policy of an optimal agent, and $\left( \pi_{E_i} \right)_{i=2}^n$ are a collection of $n$ sub-optimal policies with known degree of sub-optimality $\xi_i \in \mathbb{R}_{> 0}$.\footnote{For the sake of exposition, we consider a single optimal agent. The extension to cases where multiple optimal policies are available is direct. Futher details on this point are provided in Appendix \ref{app:multiple-expert-setting}.}  A reward function $r \in \mathbb{R}^{\mathcal{S} \times \mathcal{A}}$, is feasible for $\mathfrak{B}$ if $\pi_{E_1}$ is an optimal policy for the MDP $\mathcal{M} \cup r$ and, furthermore, if:
\begin{align}\label{eq:ass}
    \left\|V_{\mathcal{M} \cup r}^{\pi_{E_1}} - V_{\mathcal{M} \cup r}^{\pi_{E_i}}\right\|_\infty \le \xi_i,
\end{align}
holds for all $i \in \left\{2, \dots, n+1\right\}$. In this sense, $\xi_i$ (i.e., the degree of sub-optimality of policy $\pi_{E_i}$) represents a known upper bound on the performance between the optimal expert and the $i$-th sub-optimal agent. We denote by $\mathcal{R}_{\bar{\mathfrak{B}}}$ the set of feasible rewards for $\bar{\mathfrak{B}}$. More formally, $r \in [0,1]^{\mathcal{S} \times \mathcal{A}}$ belongs to $\mathcal{R}_{\bar{\mathfrak{B}}}$ if (i) $A^{\pi_{E_1}}_{\mathcal{M} \cup r} \le 0$ and (ii) Equation \eqref{eq:ass} holds for all $i \in \left\{2, \dots, n+1 \right\}$. Notice that, whenever no sub-optimal expert is present, we directly recover the definition of the feasible set for single-agent IRL problems, i.e., $\mathcal{R}_{{\mathfrak{B}}}$ in Equation \eqref{eq:feasible-set-irl}.


\paragraph{Empirical Estimates} 
Let $\mathcal{D}_t$ be a dataset of transitions of $t$ tuples $\mathcal{D}_t = \left\{ \left(s_j, a_j, s'_j, ( a^{(i)}_j )_{i=1}^{n+1} \right)   \right\}_{j=1}^{t}$, where $s_j' \sim p(\cdot|s_j,a_j)$, and $a^{(i)}_j \sim \pi_{E_i}(\cdot | s_j)$. Given $\mathcal{D}_t$, it is possible to define the empirical transition model $\hat{p}$ and the empirical experts' policy $\hat{\pi}_{E_i}$ as follows:
\begin{equation}\label{eq:updates}
\begin{aligned}
&\hat{p}(s'|s,a) = \begin{cases}
							\frac{N_t(s,a,s')}{N_t(s,a)} & \text{if } N_t(s,a) > 0\\
							\frac{1}{S} & \text{otherwise}
				   \end{cases}, \\
& \hat{\pi}_{E_i}(a|s) = \begin{cases}
							\frac{N_t^{(i)}(s,a)}{N_t(s)} & \text{if } N_t(s) > 0\\
							\frac{1}{A} & \text{otherwise}
				   \end{cases},				   
\end{aligned}
\end{equation}
where $N_t(s,a,s')$ denotes the number of times in which $(s_j,a_j,s'_j)$ is equal to $(s,a,s')$, $N_t(s,a) = \sum_{s'} N_t(s,a,s')$, $N_t(s) = \sum_{a, s'} N_t(s,a,s')$, and, finally, $N_t^{(i)}(s,a)$ counts the number of times in which $(s_j, a^{(i)}_j)$ is equal to $(s,a)$.
Given these definitions, we denote with $\widehat{\bar{\mathfrak{B}}}_t$ the empirical IRL problem that is induced by $\hat{p}$ and $\left\{ \hat{\pi}_{E_i} \right\}_{i=1}^{n+1}$. We denote with ${\mathcal{R}}_{\widehat{\bar{\mathfrak{B}}}_t}$ its corresponding feasible reward region.

\section{SUB-OPTIMAL EXPERTS AND THE FEASIBLE REWARD SET}\label{sec:reward-set}
In this section, we lay down the foundations for the problem of Inverse Reinforcement Learning in the presence of multiple and sub-optimal experts. Specifically, given the formulation introduced in Section \ref{sec:prel}, we now delve into an in-depth examination of the theoretical properties of the feasible reward set $\mathcal{R}_{\bar{\mathfrak{B}}}$. We will tackle the problem from two different perspectives. First, we present an implicit formulation of $\mathcal{R}_{\bar{\mathfrak{B}}}$ that will allow us to characterize the properties of the feasible set by means of $Q$ and $V$ function (Section \ref{sec:implicit}). Then, we will present an explicit formulation that will provide us with a precise mathematical description of $\mathcal{R}_{\bar{\mathfrak{B}}}$ (Section \ref{sec:explicit}). As we shall see, these results indicate that the presence of sub-optimal experts can significantly shrink the feasible set of compatible rewards.

\subsection{Implicit Formulation of $\mathcal{R}_{\bar{\mathfrak{B}}}$}\label{sec:implicit}
As mentioned above, we begin by providing an implicit description of the feasible reward set $\mathcal{R}_{\bar{\mathfrak{B}}}$. To this end, we derive the following result (proof in Appendix \ref{app:proofs-sec-3}). 

\begin{restatable}{lemma}{lemmaimpplicit}\label{lemma:implicit-multiple-experts}
Let $\bar{\mathfrak{B}}$ be an IRL problem with sub-optimal experts. Let $r \in [0,1]^{\mathcal{S} \times \mathcal{A}}$. Then, $r \in \mathcal{R}_{\bar{\mathfrak{B}}}$ if and only if the following conditions are satisfied:
\begin{itemize}
    \item[(i)] $Q_{\mathcal{M} \cup r}^{\pi_{E_1}}(s,a) = V_{\mathcal{M} \cup r}^{\pi_{E_1}}(s)$ \quad \quad $\forall (s,a) : \pi_{E_1}(a|s) > 0$,
    \item[(ii)] $Q_{\mathcal{M} \cup r}^{\pi_{E_1}}(s,a) \le  V_{\mathcal{M} \cup r}^{\pi_{E_1}}(s)$ \quad \quad $\forall (s,a) : \pi_{E_1}(a|s) = 0$,
    \item[(iii)] $V_{\mathcal{M} \cup r}^{\pi_{E_1}} \le V_{\mathcal{M} \cup r}^{\pi_{E_i}} + \mathbf{1}_S \xi_i$ \quad  \quad $\textrm{ }  \forall i \in \left\{2, \dots, n+1\right\}$.
\end{itemize}
\end{restatable}


Lemma \ref{lemma:implicit-multiple-experts} provides necessary and sufficient conditions for determining whether a reward function $r$ belongs to the feasible set $\mathcal{R}_{\bar{\mathfrak{B}}}$. More precisely, condition (i) and (ii) directly encodes the optimality of policy $\pi_{E_1}$ for $\mathcal{M} \cup r$, i.e., the advantage function $A^{\pi_{E_1}}_{\mathcal{M} \cup r}$ is non-positive in each state-action pair. Condition (iii), on the other hand, arises from the presence of sub-optimal experts, and it is directly related to Equation \eqref{eq:ass}.

At this point, by closely examining Lemma \ref{lemma:implicit-multiple-experts}, it is possible to gain insight into the limitations and advantages associated with the additional presence of multiple and sub-optimal experts. Consider, indeed, the following illustrative examples.

\begin{example}\label{exe:1}
Suppose that $\pi_{E_i} = \pi_{E_1}$ holds for all $i \in \left\{2,\dots, n+1 \right\}$. In this case, condition (iii) is clearly satisfied for any reward function $r$. It follows that the feasible reward set $\mathcal{R}_{\bar{\mathfrak{B}}}$ is purely determined by the requirement that the advantage function of $\pi_{E_1}$ is non-negative, and, as a consequence, the set $\mathcal{R}_{\bar{\mathfrak{B}}}$ coincides with the one of the single-expert IRL problem, namely $\mathcal{R}_{\bar{\mathfrak{B}}} = \mathcal{R}_{{\mathfrak{B}}}$. Analogously, if $\xi_i \ge (1-\gamma)^{-1}$ holds for all sub-optimal experts, condition (iii) is vacuous, and, similarly to the previous case, $\mathcal{R}_{\bar{\mathfrak{B}}}$ reduces to $\mathcal{R}_{{\mathfrak{B}}}$.
\end{example}

\begin{figure}[t]
\centering\includegraphics[width=6cm]{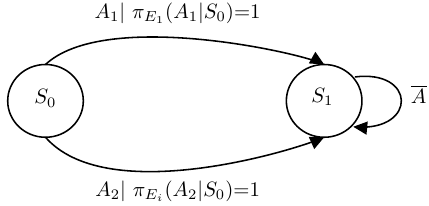} 
\vspace{.3in}
\caption{MDP example, with $2$ states and $2$ experts, that highlights the benefits of sub-optimal agents (Example \ref{exe:2}). In $S_{1}$ both $\pi_{E_1}$ and $\pi_{E_i}$ are identical, i.e., $\pi_{E_1}(\bar{A}|S_1)=\pi_{E_i}(\bar{A}|S_1)=1$.}
\label{fig:mdp}
\end{figure}

\begin{example}\label{exe:2}
Consider the MDP with $2$ states depicted in Figure~\ref{fig:mdp}, and suppose, for the sake of exposition, that only one additional sub-optimal expert is present. In this case, the optimal agent and the sub-optimal agent follows completely different policies in $S_0$. By developing the conditions in Lemma \ref{lemma:implicit-multiple-experts}, it is easy to see that, in addition to the constraint that $r(S_0, A_1) \ge r(S_0,A_2)$ (i.e., $\pi_{E_1}$ is an optimal policy), condition (iii) introduces a further relationship between $r(S_0, A_1)$ and $r(S_0,A_2)$, that is $r(S_0, A_1) - r(S_0,A_2) \le \xi_i$. In this sense, if $\xi_i$ is sufficiently small (i.e., $\xi_i < 1$ in this case), the presence of the sub-optimal agents can significantly reduce $\mathcal{R}_{\bar{\mathfrak{B}}}$ compared to $\mathcal{R}_{{\mathfrak{B}}}$.
\end{example}

Abstracting away from the previous examples, we can notice that whenever (a) the sub-optimal agents exhibit significant differences in behavior from the optimal expert and (b) their performance level is sufficiently close to being optimal, $\mathcal{R}_{\bar{\mathfrak{B}}}$ can notably shrink compared to $\mathcal{R}_{{\mathfrak{B}}}$. In the next section, through the explicit formulation of the feasible reward set, we will analyze this phenomenon quantitatively and in more detail.

\subsection{Explicit Formulation of $\mathcal{R}_{\bar{\mathfrak{B}}}$}\label{sec:explicit}

We now continue by providing an explicit formulation of the feasible set $\mathcal{R}_{\bar{\mathfrak{B}}}$. The following result (proof in Appendix \ref{app:proofs-sec-3}) summarizes our findings.

\begin{restatable}{theorem}{theoremexplicit}\label{theorem:multiple-expert-explicit}
Let $\bar{\mathfrak{B}}$ be an IRL problem with sub-optimal experts. Let $r \in [0,1]^{\mathcal{S} \times \mathcal{A}}$. Then, $r \in \mathcal{R}_{\bar{\mathfrak{B}}}$ if and only if there exists $\zeta \in \mathbb{R}^{\mathcal{S} \times \mathcal{A}}_{\ge 0}$ and $V \in \mathbb{R}^{\mathcal{S}}$ such that the following conditions are satisfied:
\begin{align}\label{eq:explicit-eq1}
    r = -\bar{B}^{\pi^{E_1}} \zeta  + (E-\gamma P)V,
\end{align}
and, for all $i \in \left\{ 2, \dots, n+1 \right\}$:
\begin{align}\label{eq:explicit-eq2}
    d^{\pi_{E_i}} \pi_{E_i}  \bar{B}^{\pi_{E_1}} \zeta  \le \mathbf{1}_{S} \xi_i.
\end{align}
\end{restatable}

Theorem \ref{theorem:multiple-expert-explicit} deserves some comments. First of all, from Equation \eqref{eq:explicit-eq1}, we can see that a necessary condition for having $r \in \mathcal{R}_{\bar{\mathfrak{B}}}$ is that it can be expressed as the sum of two different components, namely $-\bar{B}^{\pi^{E_1}} \zeta$ and $(E-\gamma P)V$. This sort of result is a direct consequence of the fact that $\pi_{E_1}$ is an optimal policy for $\mathcal{M} \cup r$, and, in this sense, it recovers the existing results of single expert IRL settings~\cite{metelli2021provably}. Indeed, we notice that it exactly matches Equation \eqref{eq:reward-char}, and, consequently, it does not depend at all on the presence of the sub-optimal experts.\footnote{We notice that, however, contrary to single-agent IRL problems, now Equation \eqref{eq:explicit-eq1} is only a necessary condition for having $r \in \mathcal{R}_{\bar{\mathfrak{B}}}$.} The role of the sub-optimal agents, on the other hand, is completely expressed by Equation \eqref{eq:explicit-eq2}.\footnote{We remark that whenever $n=1$ (i.e., we have only access to the optimal expert $\pi_{E_1}$), Theorem \ref{theorem:multiple-expert-explicit} simply reduces to Equation \eqref{eq:explicit-eq1}, and, consequently, it smoothly generalizes existing results for the classical IRL problem.} More precisely, each additional expert introduces a set of \emph{linear} constraints on the values that $\zeta$ can assume.\footnote{As a consequence of the linearity, testing whether a given $\zeta$ satisfies Equation \eqref{eq:explicit-eq2} is computationally efficient.} We recall that $-\bar{B}^{\pi_{E_1}} \zeta$ can be interpreted as the advantage function for the optimal policy $\pi_{E_1}$. In this sense, Equation \eqref{eq:explicit-eq2} limits how sub-optimal the values of actions not played by $\pi_{E_1}$ can be. Specifically, we notice that the resulting $Q$ function of the optimal expert $\pi_{E_1}$, for a given choice of $r$, can be expressed as $Q^{\pi_{E_1}}_{\mathcal{M} \cup r} = -\bar{B}^{\pi_{E_1}} \zeta + EV$ \citep{metelli2021provably}. In this sense, we can appreciate that by limiting the values of $\zeta$, we are restricting the sub-optimality gaps, expressed in terms of $Q$ functions, of actions that the optimal expert does not play. At this point, we notice that the linear constraints in Equation \eqref{eq:explicit-eq2} are expressed in terms of $\pi_{E_i}  \bar{B}^{\pi_{E_1}} \zeta$. As a consequence, they will only affect state-action pairs $(s,a)$ that are played by the sub-optimal experts (i.e., $\pi_{E_i}(a|s) > 0$) and that are not played by the optimal agent (i.e., $\pi_{E_1}(a|s) = 0$). Therefore, as previously highlighted with the implicit formulation of $\mathcal{R}_{\bar{\mathfrak{B}}}$, a sub-optimal expert $\pi_{E_i}$ should behave differently w.r.t. the optimal agent $\pi_{E_1}$ in order to provide meaningful information and reduce the feasible reward set. Furthermore, the limitations introduced over $\zeta$ are directly dependent on the expected discounted occupancy of $\pi_{E_i}$. Given these considerations, we can appreciate that Equation \eqref{eq:explicit-eq2} has provided a precise mathematical description of the phenomenon we identified at the end of the previous section.  


As a final remark, we comment on the maximum values that $\zeta$ can assume. We recall that, for classical IRL problems, $\|\zeta\|_\infty \le (1-\gamma)^{-1}$. For the sub-optimal experts case, instead, let us analyze Equation \eqref{eq:explicit-eq2} in greater detail. Fix a state $s' \in \mathcal{S}$ and a sub-optimal agent $i \in \left\{ 2, \dots, n+1\right\}$; in this case, the $s'$-th constraint in Equation \eqref{eq:explicit-eq2} can be written as:
\begin{align}\label{eq:interpretation}
\sum_{s \in \mathcal{S}} d^{\pi_{E_i}}_{s'}(s) \sum_{a: \pi_{E_1}(a|s) = 0} \pi_{E_i}(a|s) \zeta(s,a) \le \xi_i,
\end{align}
where $d^{\pi_{E_i}}_{s'}(s)$ denotes the discounted expected number of times that policy $\pi_{E_i}$ visits state $s$ starting from state $s'$. From Equation \eqref{eq:interpretation}, we can obtain necessary conditions on the values of $\zeta$ that can generate compatible reward functions. More specifically, let $\mathcal{X}(s,a) \subset \left\{ 2, \dots, n+1\right\}$ be the subset of optimal experts such that $\pi_{E_i}(a|s) > 0$. Then, for each state-action pair $(s,a)$ such that $\pi_{E_1}(a|s) = 0$ and $\pi_{E_i}(a|s) > 0$, we have that:
\begin{align}\label{eq:upper-bound}
\zeta(s,a) \le \min\left\{ k(s,a) , \frac{1}{1-\gamma}\right\} \coloneqq g(s,a),
\end{align}
where $k(s,a)$ is given by:
\begin{align}\label{eq:ksa}
k(s,a) \coloneqq \min_{i \in \mathcal{X}(s,a), s' \in \mathcal{S}} \frac{\xi_i}{d^{\pi_{E_i}}_{s'}(s)\pi_{E_i}(a|s)}.
\end{align}

More specifically, the term $k(s,a)$ directly follows from Equation \eqref{eq:interpretation}, while $(1-\gamma)^{-1}$ is the maximum value that any $\zeta(s,a)$ can assume, and arises, as in the classical IRL setting, from the fact that advantage functions are bounded by $(1-\gamma)^{-1}$ for any possible reward function. In this sense, as shown in the following example, Equation \eqref{eq:upper-bound} implies a significant potential reduction in the maximum values that the advantage function can take, i.e., how much sub-optimal, in terms of $Q$-function, an action not played by $\pi_{E_1}$ can be. 
%
\begin{example}
Consider a IRL problems with only one additional expert. Suppose that $\pi_{E_1}$ and $\pi_{E_i}$ are deterministic.  For all state-action pairs in which $\pi_{E_1}(a|s)=0$ and $\pi_{E_i}(a|s)=1$, Equation \eqref{eq:upper-bound} implies that $\zeta(s,a) \le \min\left\{\xi_i, (1-\gamma)^{-1}\right\}$. If $\xi_i$ is significantly smaller than $(1-\gamma)^{-1}$, we obtain a notable restriction on the set of feasible reward functions.
\end{example}



\section{LEARNING THE FEASIBLE SET}\label{sec:learning}
So far, we have investigated the theoretical properties of the class of reward functions that belong to the feasible set. In this section, we leverage these results to tackle the statistical complexity of estimating $\mathcal{R}_{\bar{\mathfrak{B}}}$ with a \emph{generative model}. Specifically, we first introduce a Probabilistic Approximately Correct (PAC) framework (Section \ref{sec:pac}). Then, we study the statistical complexity of the problem by presenting lower bounds on the number of samples that any algorithm requires in order to correctly identify the feasible set (Section \ref{sec:pac}). Finally, we propose a uniform sampling algorithm and analyze its theoretical guarantees (Section \ref{sec:algorithm}). As a summary, our results show that (i) the IRL problem with sub-optimal experts is statistically more demanding than the single agent IRL setting, and (ii) that the uniform sampling is minimax optimal whenever the sub-optimal experts' performance level is sufficiently close to the one of the optimal agent. For the sake of presentation, all results are presented under the assumption that $\pi_{E_1}$ is deterministic. The extension to the case in which $\pi_{E_1}$ is stochastic is presented in Appendix \ref{app:stochastic-optimal-expert}.

\subsection{PAC Framework}\label{sec:pac}
We define a learning algorithm for an IRL problem $\bar{\mathfrak{B}}$ as a tuple $\mathfrak{A} = \left(\tau, \nu \right)$, $\tau$ is a stopping time that controls the end of the data acquisition phase, and $\nu  = \left( \nu_t \right)_{t \in \mathbb{N}}$ is a history-dependent sampling strategy over $\mathcal{S} \times \mathcal{A}$. More precisely, $\nu_t \in \Delta^{\mathcal{S} \times \mathcal{A}}_{\mathcal{D}_t}$, where $\mathcal{D}_t = \left(\mathcal{S} \times \mathcal{A} \times \mathcal{S} \times \left(\mathcal{A}\right)^{n+1} \right)^{t}$ . At each time step $t \in \mathbb{N}$, the algorithm selects a state-action pair $(S_t, A_t) \sim \nu_t$, and observes a sample $S'_t \sim p(\cdot | S_t, A_t)$ from the environment, together with actions sampled from the experts' policy, namely $\left( A^{(i)}_t\right)_{i=1}^{n+1}$, where $A^{(i)}_t \sim \pi_{E_i}(\cdot | S_t)$. The observed realizations are then used to update the sampling strategy $\nu_t$, and the process goes on until the stopping rule is satisfied. At the end of the data acquisition phase, the algorithm leverages the collected data to output the estimate of the feasible reward set ${\mathcal{R}}_{\widehat{\bar{\mathfrak{B}}}_\tau}$ that is induced by the resulting empirical IRL problem $\widehat{\bar{\mathfrak{B}}}_\tau$. Given this formalism, we are interested in designing learning algorithms that, for any desired accuracy $\epsilon \in (0, 1)$ and any risk parameter $\delta \in (0, 1)$, guarantee that:
\begin{align}\label{pac:def}
\mathbb{P}_{\mathfrak{A},\bar{\mathfrak{B}}} \left( H_{\infty}(\mathcal{R}_{\bar{\mathfrak{B}}}, {\mathcal{R}}_{\widehat{\bar{\mathfrak{B}}}_\tau}) > \epsilon \right) \le \delta.
\end{align}
We refer to these algorithms as $(\epsilon,\delta)$-correct identification strategies. For $(\epsilon,\delta)$-correct strategies, we define their sample complexity as the total number of interaction rounds with the generative model before stopping. In other words, the sample complexity is given by $\tau$.

\subsection{Statistical Lower Bound}\label{sec:lower-bound}
In this section, we present lower bounds on the number of queries to the generative model that any $(\epsilon, \delta)$-correct algorithm needs to perform in order to correctly identify the feasible reward set $\mathcal{R}_{\bar{\mathfrak{B}}}$. The following theorem (proof in Appendix \ref{app:proof-app-sec-4}) reports our result.

\begin{restatable}{theorem}{lb}\label{theo:lb}
Let $\mathfrak{A}$ be a $(\epsilon, \delta)$-correct algorithm for the IRL problem with sub-optimal experts. There exists a problem instance $\bar{\mathfrak{B}}$ such that the expected sample complexity is lower bounded by:
\begin{align}\label{eq:lb-one}
	\mathbb{E}_{\mathfrak{A}, \bar{\mathfrak{B}}}[\tau] \ge {\Omega} \left(\frac{SA}{\epsilon^2 (1-\gamma)^2} \left(\log\left( \frac{1}{\delta}\right) + S \right) \right),
\end{align}
where $\Omega(\cdot)$ hides constant dependencies. Furthermore, let $\pi_\textup{min}$ be:
\begin{align}\label{eq:pi-min-def}
\pi_\textup{min} \coloneqq \min_{i \in \left\{ 2, \dots n+1 \right\}} \max_{(s,a): \pi_{E_i}(a|s) > 0} \pi_{E_i}(a|s),
\end{align}
and define $q_0 \coloneqq \pi_{\textup{min}}^{-1} {\max_{i\in\{2,\dots,n+1\}} \xi_i}$. Then there exists an instance $\bar{\mathfrak{B}}'$ in which $q_0 < 1$ such that:
\begin{align}\label{eq:lb-two}
	\mathbb{E}_{\mathfrak{A}, \bar{\mathfrak{B}}'}[\tau] \ge {\Omega} \left( \frac{q_0^2 S \log\left( \frac{1}{\delta} \right)}{\epsilon^2 \pi_\textup{min}} \right).
\end{align}
\end{restatable}

Theorem \ref{theo:lb} provides two distinct lower bounds (i.e., Equations \eqref{eq:lb-one} and \eqref{eq:lb-two}) for IRL problems with sub-optimal experts. As a consequence, we notice that whenever $q_0 < 1$ holds, the lower bound for the IRL-SE setting can be expressed as the maximum between Equation \eqref{eq:lb-one} and \eqref{eq:lb-two}. At this point, we will comment in-depth on these two equations.


Concerning Equation \eqref{eq:lb-one}, as our analysis reveals, it directly arises from the problem of estimating rewards functions that are compatible with $\pi_{E_1}$ (i.e., with Equation \eqref{eq:explicit-eq1} in Theorem \ref{theorem:multiple-expert-explicit}). In this sense, it represents the complexity of single-agent IRL problems.\footnote{We notice that similar results were presented in \citet{metelli2023towards} for the finite-horizon single expert IRL problem. In this work, we extend their construction and analysis to the infinite-horizon IRL model.} As a precise consequence of the structure of the feasible region we derived in Theorem \ref{theorem:multiple-expert-explicit}, this results in a lower bound also for the multiple sub-optimal experts setting. Therefore, Equation \eqref{eq:lb-one} formally shows that the sub-optimal expert setting is always at least as difficult as the single agent IRL problem.

Equation \eqref{eq:lb-two}, on the other hand, is strongly related to the presence of sub-optimal experts. More precisely, under the assumption that $q_0 < 1$ (e.g., for sufficiently small values of $\xi_i$), it shows a dependency in the lower bound of a factor ${\pi^{-1}_{\textup{min}}}$, where $\pi_{\textup{min}}$ represents the minimum probability with which sub-optimal experts plays their actions. From an intuitive perspective, its presence is related to the difficulty in estimating reward functions that are compatible with Equation \eqref{eq:explicit-eq2} in Theorem \ref{theorem:multiple-expert-explicit}. Indeed, as we have shown in Section \ref{sec:reward-set}, the presence of sub-optimal agents can limit the value of $\zeta$ with a relationship that involves $\pi_{\textup{min}}^{-1}$ (i.e., Equation \eqref{eq:upper-bound}). As our analysis will reveal, the proof of Equation \eqref{eq:lb-two} is directly related to these worst-case upper-bounds on $\zeta$ (and, in order to exploit them successfully, we needed to restrict ourselves to the case in which $q_0 < 1$). At this point, it has to be remarked that, according to the value of $\pi_{\textup{min}}$, Equation \eqref{eq:lb-two} can be significantly larger than Equation \eqref{eq:explicit-eq2}, thus showing an increased difficulty in the statistical complexity that is related to the stochasticity of sub-optimal experts. 

At this point, it has to be noticed that the generative model we defined in Section \ref{sec:pac} is significantly more powerful than the one adopted in a classical IRL setting \citep[see, e.g.,][]{metelli2021provably,metelli2023towards}. For single-agent problems, indeed, a query to the generative model provides only samples from the environment and from the expert agent $\pi_{E_1}$. In our context, on the other hand, for each query, the generative model provides demonstrations from \emph{each} sub-optimal expert. It can be shown that, by slightly modifying the learning formalism, Equation \eqref{eq:lb-two} actually represents a lower bound to the number of samples that should be gathered from \emph{each} sub-optimal agent.\footnote{For further details on this point, we defer the reader to Appendix \ref{app:new-learning-formalism}.} In this sense, the statistical complexity increases significantly in the sub-optimal expert setting compared to the single agent one. Therefore, as a concluding remark, we notice that, in order to gain the reduction in the feasible reward set that we discussed in Section \ref{sec:related-works}, we need to gather additional data in terms of demonstrations from the sub-optimal experts. This unavoidable trade-off is a direct consequence of the structure of the feasible set $\mathcal{R}_{\hat{\mathfrak{B}}}$ that we derived in Theorem \ref{theorem:multiple-expert-explicit}, and, indeed, it arises from the statistical complexity of estimating reward functions that are compatible with the linear constraints of Equation \eqref{eq:explicit-eq2}.


\subsection{Uniform Sampling Algorithm}\label{sec:algorithm}
In this section, we present the Uniform Sampling algorithm for Inverse RL with Suboptimal Experts (US-IRL-SE). The pseudo-code can be found in Algorithm \ref{alg:usirl}. As we can see, US-IRL-SE receives the number of samples $m$ that will be queried to the generative model in each state-action pair. Then, it uniformly gathers data across the entire state-action space, and it updates the empirical estimates $\hat{p}$ and $(\hat{\pi}_{E_i} )_{i=1}^{n+1}$. 

\begin{algorithm}[t]
\caption{Uniform Sampling for Inverse RL with Suboptimal Experts (US-IRL-SE)} \label{alg:usirl}
\small
\begin{algorithmic}[1]
\REQUIRE{samples collected in each $(s,a)$ pair $m$}
\FOR{$t=1,2,\dots, m$}
\STATE{Collect one tuple $(s',(a^{(i)})_{i=1}^{n+1})$  where $s' \sim p(\cdot|s,a)$ and $a^{(i)} \sim \pi_{E_i}(\cdot|s)$ from each $(s,a) \in \mathcal{S} \times \mathcal{A}$ and $i \in \{1\,\dots,n+1\}$}
\STATE{Update $\hat{p}$ and $\left( \hat{\pi}_{E_i} \right)_{i=1}^{n+1}$ according to Equation~\eqref{eq:updates}}
\ENDFOR
\end{algorithmic}
\end{algorithm}

The following theorem (proof in Appendix \ref{app:proof-app-sec-4}), describes the theoretical guarantees of US-IRL-SE.
\begin{restatable}{theorem}{upperbound}\label{theo:ub}
Let $q_1 = \min \left\{ \pi_\textup{min}^{-1} \max_{i } \xi_i  ,(1-\gamma)^{-1} \right\}$, and $q_2 = \max \left\{ 1, q_1 \right\}$. Then, with a total budget of:
\begin{align}\label{eq:ub}
\widetilde{\mathcal{O}}\left( \max\left\{ \frac{q_1^2 S \log \left( \frac{1}{\delta}\right)}{\pi_\textup{min} \epsilon^2},   \frac{q_2^2 SA (S + \log\left( \frac{1}{\delta} \right))}{\epsilon^2 (1-\gamma)^{2}}  \right\} \right),
\end{align}
US-IRL-SE is $(\epsilon,\delta)$-correct and $\widetilde{\mathcal{O}}\left( \cdot \right)$ hides constant and logarithmic dependencies.
\end{restatable}

Theorem \ref{theo:ub} deserves some comments. First of all, it formally shows that when the total number of queries to the generative is sufficiently large, US-IRL-SE is $(\epsilon, \delta)$-correct, and its sample complexity is provided in Equation \eqref{eq:ub}. In this sense, we notice that, since $m$ represents the number of calls to the generative model in each state-action pair, its expression can simply be calculated by dividing Equation \eqref{eq:ub} by $SA$.\footnote{The exact expression of $m$ (i.e., constants and hidden logarithmic factors) is provided in Appendix \ref{app:details-us-irl-se}.} As a consequence, we remark that, in order to compute the value of $m$, the algorithm requires knowledge of the minimum probability with which sub-optimal experts play their actions. 

We now proceed by analyzing in detail the sample complexity guarantee. Equation \eqref{eq:ub} is the maximum between two terms whose expressions closely resemble the lower bound that we presented in Theorem \ref{theo:lb}. Specifically, the only difference arises in the definition of $q_0$, $q_1$ and $q_2$. Currently, we are unsure whether this gap arises from the lower bound or the algorithm analysis, and we leave this gap to be filled in for future work. Nevertheless, it has to be remarked that, whenever the sub-optimal expert's performance level is sufficiently close to the one of the optimal agent (i.e., $\pi_{\textup{min}}^{-1} \xi_i \le 1$ for all $i \in \left\{2, \dots, n+1 \right\}$), Equation \eqref{eq:ub} exactly recovers the lower bound that we presented in Theorem \ref{theo:lb}.\footnote{More precisely, under the condition that $\pi_{\textup{min}}^{-1} \xi_i \le 1$, it holds that $q_0 = q_1$, and $q_2 = 1$.} We remark that according to Theorem \ref{theorem:multiple-expert-explicit}, as the values of $\xi_i$'s decrease, the feasible reward set is substantially reduced. In this sense, US-IRL-SE enjoys minimax optimality in the most interesting scenarios where the presence of sub-optimal experts is particularly valuable for mitigating the intrinsic ambiguity that affects inverse reinforcement learning problems. 



\paragraph{Technical Remark} To conclude, we highlight that, although the algorithm is relatively simple, the proof of Theorem \ref{theo:ub} requires significant technical effort. The main challenge arises from studying how the Hausdorff distance between $\mathcal{R}_{\bar{\mathfrak{B}}}$ and $\mathcal{R}_{\widehat{\bar{\mathfrak{B}}}_t}$ decreases as we collect more data from the generative model. Indeed, we recall that these feasible reward sets are subject to the peculiar structure that we identified in Theorem \ref{theorem:multiple-expert-explicit}. More specifically, the set of constraints of Equation \eqref{eq:explicit-eq2} that arises from the presence of sub-optimal experts complicates significantly the study of $H_\infty\left(\mathcal{R}_{\bar{\mathfrak{B}}}, {\mathcal{R}}_{\widehat{\bar{\mathfrak{B}}}_t} \right)$. For further details on this point, we invite the reader to consult our proofs Appendix \ref{app:proof-app-sec-4}.

\section{RELATED WORKS}\label{sec:related-works}

\paragraph{Inverse Reinforcement Learning} Historically, solving an IRL problem \citep{adams2022survey} involves determining a reward function that is compatible with the behavior of an optimal expert. Since the seminal work of \citet{ng2000algorithms}, the problem has been recognized as ill-posed, as multiple reward functions that satisfies this requirement exists \citep{invariance2023}. For this reason, over the years, several algorithmic criteria have been introduced to address this ambiguity issue. These criteria includes maximum margin \citep{ratliff2006maximum}, Bayesian approaches \citep{ramachandran2007bayesian}, maximum entropy \citep{ziebart2008maximum}, and many others \citep[e.g.,][]{majumdar2017risk,metelli2017compatible,zeng2022maximum}. More recently, a new line of works have circumvented the ambiguity issue by redefining the IRL task as the problem of estimating the entire feasible reward set \citep{metelli2021provably,lindner2022active,metelli2023towards}. In our work, we take this novel perspective, and, in this sense, this recent research strand is the most related to our document. Specifically, of particular interests is the work of \citet{metelli2023towards}. In their work, the authors study, for the first time, lower bounds for the single-agent IRL problem in finite horizon settings; furthermore, they show that uniform sampling algorithm is minimax optimal for this task. Nevertheless, it has to remarked that this recent strand of research focuses entirely on single expert problems. As we have shown, however, the extension to the multiple and sub-optimal agents setting requires non-trivial effort. Indeed, the feasible reward set significantly differ (see, e.g., Theorem \ref{theorem:multiple-expert-explicit}), and the problem is harder from a statistical perspective (see, e.g., Theorem \ref{theo:lb}).

\paragraph{Multiple and/or Sub-optimal Experts}
The presence of multiple/sub-optimal experts has garnered attention in the Imitation Learning \citep[IL,][]{hussein2017imitation} community. In IL problems, contrary to IRL, the goal lies in directly leveraging demonstrations of optimal behavior to accelerate the training process of reinforcement learning algorithms. In this context, works that are close in spirit to ours are \citet{kurenkov2020ac,jing2020reinforcement,cheng2020policy,liu2023active}; here, the authors extends the IL formulation to account for the fact that demonstrations are provided from multiple and/or sub-optimal experts. However, unlike our specific focus, their emphasis is on understanding how to effectively exploit imperfect demonstrations to improve training of RL agents. In our work, instead, we exploit the presence of sub-optimal agents to reduce the intrinsic ambiguity that affects the IRL formulation. In this sense, our work is complementary to several studies that analyzed how to improve the identifiability of the reward function in IRL problems by making additional structural assumptions. These include the possibility of observing an optimal agent interacting with several MDPs \citep[e.g.,][]{ratliff2006maximum,amin2016towards,amin2017repeated} and focusing on peculiar types of MDPs that allows for strong theoretical guarantees \citep[e.g.,][]{dvijotham2010inverse,kim2021reward,cao2021identifiability}. Along this line of work, the most related to ours is \citet{rolland2022identifiability}. Here, the authors study how the presence of multiple experts impact the identifiability of the reward function. Contrary to our work, however, the authors assume each agent to follow an entropy regularized objective and, furthermore, they focus on the case in which all experts act optimally in the underlying environment. In this sense, our work encompasses a wider spectrum of applications, as we do not require optimality for each of the agent, nor an entropy regularized objective. Finally, it has to be remarked that the multiple expert setting and IRL have been studied in \citet{likmeta2021dealing} with the goal of providing practical algorithms that can be used in real-world applications. Also in this scenario, each agent is assumed to act optimally in the underlying domain.

\section{CONCLUSIONS}\label{sec:conclusions}
In this work, we studied the novel problem of Inverse RL where, in addition to demonstrations from an optimal expert, we can observe the behavior of multiple and sub-optimal agents. More precisely, we first investigated the theoretical properties of the class of reward functions that are compatible with a given set of experts, i.e., the feasible reward set. Our results formally show that, by exploiting this additional structure, it is possible to significantly reduce the intrinsic ambiguity that affects the IRL formulation. Secondly, we have tackled the statistical complexity of estimating the feasible reward set from a generative model. More precisely, we have shown that a uniform sampling algorithm is minimax optimal whenever the performance level of the sub-optimal expert is sufficiently close to the one of the optimal agent. 

Our research opens up intriguing avenues for future studies. For instance, since we have shown that sub-optimal experts can improve the identifiability of the reward function, future research should focus on building practical algorithms that can exploit this additional structure. To this end, as an intermediate step, it might be interesting to extend our results to the case in which the reward function is expressed as a linear combination of features. This approach would enable addressing infinite state-spaces \citep[e.g.,][]{ng2000algorithms}.

\bibliography{biblio} 


\onecolumn
\aistatstitle{Inverse Reinforcement with Sub-optimal Experts: \\
Supplementary Materials}

\appendix

The structure of the supplementary materials is organized as follows:
\begin{itemize}
\item {Appendix \ref{app:multiple-expert-setting} describes how to extend our result to the multiple optimal expert setting.}
\item {Appendix \ref{app:proofs-sec-3} provides formal proofs of the theoretical claims of Section \ref{sec:reward-set}}.
\item {Appendix \ref{app:proof-app-sec-4} provides formal proofs of the theoretical claims of Section \ref{sec:learning}.}
\item {Appendix \ref{app:stochastic-optimal-expert} discusses how to extend our result to the setting in which $\pi_{E_1}$ is stochastic.}
\item {Appendix \ref{app:new-learning-formalism} discusses how to change the learning formalism to derive a lower-bound that directly depends on the number of samples that are needed from each sub-optimal expert.}
\item {Appendix \ref{app:details-us-irl-se} provides additional details on US-IRL-SE (i.e., exact description of $m$) and computational complexity analysis.}
\end{itemize}

To begin, we provide tables that summaries the main symbols used in this document.
Table \ref{table:notation} reports a summary on the notation used throughout the paper. Table \ref{table:operators} reports a precise definition of the operators used.

\begin{table}[h]
\caption{Notation}\label{table:notation}
\begin{center}
\begin{tabular}{ll}
\textbf{SYMBOL}  &\textbf{MEANING} \\
\hline \\
$\mathfrak{B}$         & Inverse RL problem with a single optimal agent. \\
$\bar{\mathfrak{B}}$             & Inverse RL problem with multiple and sub-optimal experts. \\
$n$								  & Number of sub-optimal experts, $n \in \mathbb{N}_{>0}$. \\
$\pi_{E_i}$             & Policy of the $i$-th expert. If $i=1$, the expert is optimal. \\
$\xi_i$             & Sub-optimality of the $i$-th expert, where $i \in \left\{2, \dots, n+1 \right\}$. \\
$\hat{\pi}_{E_i}$ 					& Empirical estimates of the $i$-th expert policy $\pi_{E_i}$. \\
$\hat{p}$ 					& Empirical estimate of the transition model $p$.\\
$\mathcal{D}_t$            & Dataset of $t$ tuples from the generative model. \\
$N_t(s,a)$					& Number of samples gathered at state-action pair $(s,a)$ in $\mathcal{D}_t$. \\
$N_t(s)$					& Number of samples gathered at state $s$ in $\mathcal{D}_t$. \\
$\widehat{\bar{\mathfrak{B}}}$ & Empirical estimate of the IRL-SE problem induced by $\hat{p}$ and $\hat{\pi}_{E_i}$. \\
$\mathcal{R}_{\mathfrak{B}}$ & Feasible reward set of a single-agent IRL problem. \\
$\mathcal{R}_{\bar{\mathfrak{B}}}$ & Feasible reward set of a IRL-SE problem. \\ 
$\mathcal{R}_{\widehat{\bar{\mathfrak{B}}}}$ & Feasible reward set of the IRL-SE problem induced by $\hat{p}$ and $\hat{\pi}_{E_i}$.\\
$H_\infty(\mathcal{X}, \mathcal{X}')$ & Hausdorff distance between set $\mathcal{X}$ and $\mathcal{X}'$.\\
$\mathfrak{A}$ & Learning algorithm for the US-IRL-SE problem.\\
$\nu$ & Sampling strategy of a learning algorithm $\mathfrak{A}$. \\
$\tau$ & Stopping time (i.e., sample complexity) of a learning algorithm $\mathfrak{A}$. \\
$\epsilon$ & Desired level of accuracy when estimating the feasible reward set. \\
$\delta$ & Maximum risk tolerated when estimating the feasible reward set. \\
$m$ & Number of samples that the US-IRL-SE algorithm gathers in each state-action pair. \\
\end{tabular}
\end{center}
\end{table}

\begin{table}[h]
\caption{Operators}\label{table:operators}
\begin{center}
\begin{tabular}{lll}
\textbf{SYMBOL}  & \textbf{SIGNATURE} &\textbf{DEFINITION} \\
\hline \\
$P$ & $\mathbb{R}^{S} \rightarrow \mathbb{R}^{S \times A}$ & $(Pf)(s,a) = \sum_{s' \in S} p(s'|s,a) f(s')$ \\
$\pi$ & $\mathbb{R}^{S \times A} \rightarrow \mathbb{R}^S$ & $(\pi f)(s) = \sum_{a \in \mathcal{A}} \pi(a|s) f(s,a)$\\
$E$ & $\mathbb{R}^{S} \rightarrow \mathbb{R}^{S \times A}$ & $(Ef)(s,a) = f(s)$\\
$\bar{B}^\pi$ & $\mathbb{R}^{S \times A} \rightarrow \mathbb{R}^{S \times A}$ & $(\bar{B}^{\pi}f)(s,a) = \mathbbm{1} \left\{ \pi(a|s)=0\right\} f(s,a)$ \\
${B}^\pi$ & $\mathbb{R}^{S \times A} \rightarrow \mathbb{R}^{S \times A}$ & $({B}^{\pi}f)(s,a) = \mathbbm{1} \left\{ \pi(a|s)>0\right\} f(s,a)$ \\ 
$d^\pi$ & $\mathbb{R}^{S} \rightarrow \mathbb{R}^{S}$ & $(d^{\pi} f)(s) = \sum_{t=0}^{+\infty} \left((\gamma \pi P)^t f\right)(s)$ \\ 
${I}_{\mathcal{S}}$ & $\mathbb{R}^{S} \rightarrow \mathbb{R}^{S}$ & $({I}_{\mathcal{S}}f)(s) = f(s)$
\end{tabular}
\end{center}
\end{table}

\section{Additional Multiple Optimal Expert Setting}\label{app:multiple-expert-setting}
In this section, we discuss the extension of the IRL-SE setting to the case in which multiple optimal policies are available. More specifically, we define this IRL-SE setting as a tuple $\tilde{\mathfrak{B}} = \left( \mathcal{M}, \left( \pi_{E_i}^*\right)_{i=1}^{n_1}, \left( \pi_{E_i} \right)_{i=1}^{n_2}, \left( \xi_i \right)_{i=1}^{n_2} \right)$, where $\left( \pi_{E_i}^*\right)_{i=1}^{n_1}$ is a set of $n_1$ optimal policies, $\left( \pi_{E_i} \right)_{i=1}^{n_2}$ is a set of $n_2$ sub-optimal policies with known degree of sub-optimality $\left( \xi_i \right)_{i=1}^{n_2}$.

At this point, it is easy to verify that the shape of the feasible set $\mathcal{R}_{\tilde{\mathfrak{B}}}$, is exactly the one of Theorem \ref{theorem:multiple-expert-explicit}, where Equations \eqref{eq:explicit-eq1} an \eqref{eq:explicit-eq2} are obtained by replacing $\pi_{E_1}$ with each optimal policy $\pi_{E_i}^{*}$ with $i \in \left\{1, \dots, n_1 \right\}$.

Concerning learning the feasible reward set, it is sufficient to extend the generative model so that samples are gathered from all optimal and sub-optimal experts. The proof of Section \ref{sec:learning} holds almost unchanged.

\section{Proofs and Derivations of Section \ref{sec:reward-set}}\label{app:proofs-sec-3}

In this section, we provide formal proofs of the theoretical results of Section \ref{sec:reward-set}. We begin by reporting for completeness some results from \citet{metelli2021provably} that will be used in our analysis.

\begin{lemma}\label{lemma:metelli-implicit}
Let $ \mathfrak{B} = (\mathcal{M}, \pi_E)$ be a single-agent IRL problem. Let $r \in [0,1]^{S\times A}$, then $r$ is a feasible reward if and only if for all $(s,a) \in (S,A)$ it holds that:
\begin{align*}
    (i) \qquad Q^{\pi_E}_{\mathcal{M} \cup r}(s,a) - V^{\pi_E}_{\mathcal{M} \cup r}(s) = 0 \qquad if \ \pi_E(a|s) > 0,  \\
    (ii) \qquad Q^{\pi_E}_{\mathcal{M} \cup r}(s,a) - V^{\pi_E}_{\mathcal{M} \cup r}(s) \leq 0 \qquad if \ \pi_E(a|s) = 0.
\end{align*}
\end{lemma}

\begin{lemma}\label{lemma:metelli-q-funct}
Let $ \mathfrak{B} = (\mathcal{M}, \pi_E)$ be a single-agent IRL problem. A Q-function satisfies condition of Lemma \ref{lemma:metelli-implicit} if and only if there exist $\zeta \in \mathbb{R} ^ {S \times A}_{\geq 0}$ and $V\in \mathbb{R}^S$ such that:
\begin{align*}
    Q_{\mathcal{M} \cup r} = - \bar B^{\pi_E} \zeta + EV
\end{align*}
\end{lemma}

\begin{lemma}\label{lemma:metelli-explicit}
Let $ \mathfrak{B} = (\mathcal{M}, \pi_E)$ be a single-agent IRL problem. Let $r \in \mathbb R ^{S \times A}$, then r is a feasible reward, if and only if there exist $\zeta \in \mathbb{R} ^ {S\times A}_{\geq 0}$ and $V\in \mathbb{R}^S$ such that:
\begin{align*}
    r = - \bar B^{\pi_E} \zeta + (E-\gamma P)V
\end{align*}
\end{lemma}

At this point, we proceed by proving Lemma \ref{lemma:implicit-multiple-experts} of Section \ref{sec:related-works}.

\lemmaimpplicit*
\begin{proof}
Condition (i) and (ii) are necessary conditions for the claim to hold. This directly follows by the definition of $\mathcal{R}_{\bar{\mathfrak{B}}}$ and from Lemma \ref{lemma:metelli-implicit}. At this point, what remains to be proven is condition (iii). Condition (iii), however, is a direct consequence of the structural assumption of the sub-optimal experts, which concludes the proof.
\end{proof}

We now continue by proving a stronger version of Lemma \ref{lemma:implicit-multiple-experts} that studies the presence of sub-optimal experts in terms of $Q$-function.

\begin{lemma}\label{lemma:implicit-multiple-experts-q}
Let $\bar{\mathfrak{B}}$ be an IRL problem with sub-optimal experts. Let $r \in [0,1]^{\mathcal{S} \times \mathcal{A}}$. Then, $r \in \mathcal{R}_{\bar{\mathfrak{B}}}$ if and only if the following conditions are satisfied:
\begin{itemize}
    \item[(i)] $Q_{\mathcal{M} \cup r}^{\pi_{E_1}}(s,a) = V_{\mathcal{M} \cup r}^{\pi_{E_1}}(s)$ \quad \quad $\forall (s,a) : \pi_{E_1}(a|s) > 0$,
    \item[(ii)] $Q_{\mathcal{M} \cup r}^{\pi_{E_1}}(s,a) \le  V_{\mathcal{M} \cup r}^{\pi_{E_1}}(s)$ \quad \quad $\forall (s,a) : \pi_{E_1}(a|s) = 0$,
    \item[(iii)] $Q_{\mathcal{M} \cup r}^{\pi_{E_1}} \le V_{\mathcal{M} \cup r}^{\pi_{E_i}} + \mathbf{1}_S \xi_i$ \quad  \quad $\textrm{ }  \forall i \in \left\{2, \dots, n+1\right\}$.
\end{itemize}
\end{lemma}
\begin{proof}
Condition (i) and (ii) are necessary conditions for the claim to hold. This directly follows by the definition of $\mathcal{R}_{\bar{\mathfrak{B}}}$ and from Lemma \ref{lemma:metelli-implicit}. At this point, what remains to be proven is condition (iii). 

We begin by showing that, if $r \in \mathcal{R}_{\mathfrak{B}}$ then condition (iii) is satisfied. Since (i) and (ii) holds, it sufficies to plug $V_{\mathcal{M} \cup r}^{\pi_{E_1}} - V_{\mathcal{M} \cup r}^{\pi_{E_i}} \le \xi_i$ within (i) and (ii) to obtain the desired result.

We conclude by showing that if (i)-(iii) are satisfied, then $r \in \mathcal{R}_{\mathfrak{B}}$. Since (i) and (ii) holds, then, by Lemma \ref{lemma:metelli-implicit}, we know that $\pi_{E_1}$ is optimal for $\mathcal{M} \cup r$. It remains to be proven that $V_{\mathcal{M} \cup r}^{\pi_{E_1}} - V_{\mathcal{M} \cup r}^{\pi_{E_i}} \le \xi_i$ holds as well. 
However, since (i) holds, (iii) implies that $V_{\mathcal{M} \cup r}^{\pi_{E_1}}(s) - V_{\mathcal{M} \cup r}^{\pi_{E_i}}(s) \le \xi_i(s)$ holds for all $s \in \mathcal{S}$, thus concluding the proof. 
\end{proof}

We continue by proving Theorem \ref{theorem:multiple-expert-explicit}.

\theoremexplicit*
\begin{proof}
	First of all, Equation \eqref{eq:explicit-eq1} is a necessary condition for the claim to hold. Indeed, from Lemma \ref{lemma:metelli-explicit} and Lemma \ref{lemma:implicit-multiple-experts-q}, Equation \eqref{eq:explicit-eq1} is a necessary and sufficient condition to make $\pi_{E_1}$ optimal in $\mathcal{M} \cup r$.

    At this point, we proceed conditioning on the fact that Equation \eqref{eq:explicit-eq1} holds. We need to show that, if $r \in \mathcal{R}_{\bar{\mathfrak{B}}}$, then Equation \eqref{eq:explicit-eq2} holds, and if Equation \eqref{eq:explicit-eq2} holds, then $r \in \mathcal{R}_{\bar{\mathfrak{B}}}$.

    From Lemma \ref{lemma:implicit-multiple-experts-q}, we know that a necessary and sufficient condition for the previous statements to hold is that:
    \begin{align*}
        Q_{\mathcal{M} \cup r}^{\pi_{E_1}}(s,a) \le V_{\mathcal{M} \cup r}^{\pi_{E_i}}(s) + \xi_i,
    \end{align*}
    holds for any state-action pair $(s,a)$ and for all experts $i \in \left\{ 2, \dots, n+1 \right\}$.
    Given Equation \eqref{eq:explicit-eq1} and Lemma \ref{lemma:metelli-q-funct}, the previous equation can be conveniently written as:
    \begin{align}\label{eq:explicit-eq3}
        -\bar{B}^{\pi^{E_1}} \zeta + EV \le EV_{\mathcal{M} \cup r}^{\pi^{E_i}} + E\mathbf{1}_{S}\xi_i.  
    \end{align}
    At this point, consider $(s,a)$ such that $\pi_{E_1}(a|s) > 0$. Then, Equation \eqref{eq:explicit-eq3} reduces to:
    \begin{align}\label{eq:explicit-eq4}
        V(s) \le V_{\mathcal{M} \cup r}^{\pi_{E_i}}(s) + \xi_i.
    \end{align}
    Conversely, consider $(s,a)$ such that $\pi_{E_1}(a|s)=0$, then Equation \eqref{eq:explicit-eq3} reduces to:
    \begin{align}\label{eq:explicit-eq5}
        -\zeta(s,a) + V(s) \le V_{\mathcal{M} \cup r}^{\pi_{E_i}}(s) + \xi_i.
    \end{align}
    Since $\zeta(s,a) \ge 0$ by assumption, Equation 
    \eqref{eq:explicit-eq5} is directly implied by Equation \eqref{eq:explicit-eq4}. Therefore, it suffices to study:
    \begin{align*}
        V(s) \le V_{\mathcal{M} \cup r}^{\pi_{E_i}}(s) + \xi_i,
    \end{align*}
    which can be rewritten as:
    \begin{align*}
        V - \mathbf{1}_S\xi_i \le V_r^{\pi_{E_i}} = \left(I_{\mathcal{S}} - \gamma \pi_{E_i} P \right)^{-1} \pi_{E_i} r. 
    \end{align*}

    At this point, we can plug Equation \eqref{eq:explicit-eq1} within the previous Equation. More specifically, we obtain:
    \begin{align}\label{eq:explicit-eq6}
        V - \mathbf{1}_S \xi_i & \le \left(I_{\mathcal{S}} - \gamma \pi_{E_i} P \right)^{-1} \pi_{E_i} \left( -\bar{B}^{\pi_{E_1}} \zeta + \left(E-\gamma P\right)V \right).
    \end{align}
    We now proceed by rewriting the right hand side of Equation \eqref{eq:explicit-eq6}. More precisely, we notice that:
    \begin{align*}
    \left(I_{\mathcal{S}} - \gamma \pi_{E_i} P \right)^{-1} \pi_{E_i} \left( \left(E-\gamma P\right)V \right) & = \left(I_{\mathcal{S}} - \gamma \pi_{E_i} P \right)^{-1} V - \gamma \left(I_{\mathcal{S}} - \gamma \pi_{E_i} P \right)^{-1} \pi_{E_i}PV \\ & = \left(I_{\mathcal{S}} - \gamma \pi_{E_i} P \right)^{-1} (I_\mathcal{S} - \gamma \pi_{E_i}P)V \\ & = V.
    \end{align*}
    Plugging this result within Equation \eqref{eq:explicit-eq6}, we arrive at:
    \begin{align*}
        \left(I_{\mathcal{S}} - \gamma \pi_{E_i} P \right)^{-1} \pi_{E_i} \left( \bar{B}^{\pi_{E_1}} \zeta \right) \le \mathbf{1}_S \xi_i,
    \end{align*} 
    which concludes the proof.
\end{proof}

\subsection{Other Assumptions on the Behavior of the Sub-optimal Experts}
In this section, we investigate the generality of the results presented in Theorem \ref{theorem:multiple-expert-explicit}. Specifically, we remark that all the results that we derived in Section \ref{sec:reward-set} can be easily extended to other assumptions on the sub-optimal experts. More specifically, the results of Theorem \ref{theorem:multiple-expert-explicit} can easily be extended to the following cases:
\begin{align}\label{ass:inequality}
V_{\mathcal{M} \cup r}^{\pi_{E_1}} - V_{\mathcal{M} \cup r}^{\pi_{E_i}} \ge \bm{1}_\mathcal{S} \xi_i,
\end{align}
or 
\begin{align}\label{ass:equality}
V_{\mathcal{M} \cup r}^{\pi_{E_1}} - V_{\mathcal{M} \cup r}^{\pi_{E_i}} = \bm{1}_\mathcal{S} \xi_i.
\end{align}
Equation \eqref{ass:inequality} encodes the fact that that a given sub-optimal expert $i$ is at least $\xi_i$ sub-optimal w.r.t. the optimal policy $\pi_{E_1}$, while Equation \eqref{ass:equality} encodes the the fact that the sub-optimal expert $i$ is exactly $\xi_i$ sub-optimal w.r.t. the optimal agent.
In these cases, it is possible to derive the following generalizations of Theorem \ref{theorem:multiple-expert-explicit}.

\begin{theorem}\label{theorem:multiple-expert-explicit-inequality}
Let $\bar{\mathfrak{B}}$ be an IRL problem with sub-optimal experts where $V_{\mathcal{M} \cup r}^{\pi_{E_1}} - V_{\mathcal{M} \cup r}^{\pi_{E_i}} \ge \bm{1}_\mathcal{S} \xi_i$ holds for all sub-optimal experts $i$. Let $r \in [0,1]^{\mathcal{S} \times \mathcal{A}}$. Then, $r \in \mathcal{R}_{\bar{\mathfrak{B}}}$ if and only if there exists $\zeta \in \mathbb{R}^{\mathcal{S} \times \mathcal{A}}_{\ge 0}$ and $V \in \mathbb{R}^{\mathcal{S}}$ such that the following conditions are satisfied:
\begin{align}
    r = -\bar{B}^{\pi^{E_1}} \zeta  + (E-\gamma P)V,
\end{align}
and, for all $i \in \left\{ 2, \dots, n+1 \right\}$:
\begin{align}
    d^{\pi_{E_i}} \pi_{E_i}  \bar{B}^{\pi_{E_1}} \zeta  \ge \mathbf{1}_{S} \xi_i.
\end{align}
\end{theorem}

\begin{theorem}\label{theorem:multiple-expert-explicit-equality}
Let $\bar{\mathfrak{B}}$ be an IRL problem with sub-optimal experts where $V_{\mathcal{M} \cup r}^{\pi_{E_1}} - V_{\mathcal{M} \cup r}^{\pi_{E_i}} = \bm{1}_\mathcal{S} \xi_i$ holds for all sub-optimal experts $i$. Let $r \in [0,1]^{\mathcal{S} \times \mathcal{A}}$. Then, $r \in \mathcal{R}_{\bar{\mathfrak{B}}}$ if and only if there exists $\zeta \in \mathbb{R}^{\mathcal{S} \times \mathcal{A}}_{\ge 0}$ and $V \in \mathbb{R}^{\mathcal{S}}$ such that the following conditions are satisfied:
\begin{align}
    r = -\bar{B}^{\pi^{E_1}} \zeta  + (E-\gamma P)V,
\end{align}
and, for all $i \in \left\{ 2, \dots, n+1 \right\}$:
\begin{align}\label{eq:stronger-eq1}
    d^{\pi_{E_i}} \pi_{E_i}  \bar{B}^{\pi_{E_1}} \zeta = \mathbf{1}_{S} \xi_i.
\end{align}
\end{theorem}

The proofs of Theorem \ref{theorem:multiple-expert-explicit-inequality} and \ref{theorem:multiple-expert-explicit-equality} are identical to the one of Theorem \ref{theorem:multiple-expert-explicit}. In terms of results, the only difference lies in the fact that the set of linear constraints introduces a different type of relationship between $\zeta$ and $\xi_i$.\footnote{As a consequence of these results, we notice that it is direct to extend the properties of the feasible reward set to the case in which, e.g., for some experts it holds that $V_{\mathcal{M} \cup r}^{\pi_{E_1}} - V_{\mathcal{M} \cup r}^{\pi_{E_i}}\le \xi_i$, while for other experts it holds that $V_{\mathcal{M} \cup r}^{\pi_{E_1}} - V_{\mathcal{M} \cup r}^{\pi_{E_i}} \ge \xi_i$.} 

At this point, we remark that Theorem \ref{theorem:multiple-expert-explicit-inequality} follows a very similar interpretation of the one we presented in Section \ref{sec:reward-set}. Concerning Theorem \ref{theorem:multiple-expert-explicit-equality}, instead, we notice that the result is much more stronger w.r.t. to the case in which inequalities are involved in the problem (i.e., Theorem \ref{theorem:multiple-expert-explicit} and Theorem \ref{theorem:multiple-expert-explicit-inequality}). More specifically, we notice that in this case, starting from Equation \eqref{eq:stronger-eq1}, it is possible to obtain the following result on the values of $\zeta$:
\begin{align}
\pi_{E_i}  \bar{B}^{\pi_{E_1}} \zeta = (I_{\mathcal{S}} -\gamma \pi_{E_i} P )\mathbf{1}_{S} \xi_i.
\end{align}
Developing this constraint for a specific state $s$, we obtain the following linear constraint:
\begin{align}\label{eq:stronger-2}
\sum_{a: \pi_{E_1}(a|s) = 0} \pi_{E_i}(a|s) \zeta(s,a) = \xi_i (1 - \gamma).
\end{align}
In other words, it provides a set of hard constraints that the values of $\zeta$ should satisfy. 
At this point, we remark on two important observations. The first one is that Equation \eqref{eq:stronger-2} might not be satisfied for any choice of $\zeta \in \mathbb{R}^{S \times A}$. In particular, suppose that there is only one sub-optimal expert. In this case, if $\pi_{E_1} = \pi_{E_i}$, Equation \eqref{eq:stronger-2} reduces to:
\begin{align*}
0 = \xi_i (1-\gamma),
\end{align*}
which is clearly false for any strictly positive value of $\xi_i$. We notice that this result is as expected. If the two policies are identical, then there should be no gap in the performance of $\pi_{E_1}$ and $\pi_{E_i}$. As a consequence, the feasible reward set of this IRL problem is empty.\footnote{Notice that, to obtain an empty feasible region, it is sufficient that the two experts behave identically in a single state-action pair. This is a direct consequence of the fact that the sub-optimality constraint is imposed with equality for each state-action pair.}
Secondly, instead, suppose that all the experts are deterministic, and suppose that they all behave differently to $\pi_{E_1}$ in each state-action pair (so that the feasible reward set is non-empty). Then, focus on the $i$-th expert and consider a state-action pair $(s,a)$ such that $\pi_{E_i}(a|s)=1$. Then, Equation \eqref{eq:stronger-2} reduces to:
\begin{align*}
\zeta(s,a) = \xi_i (1-\gamma),
\end{align*}
In this sense, the presence of the sub-optimal expert implies a unique value that $\zeta(s,a)$ can take.\footnote{Notice, in this sense, that, if any other deterministic sub-optimal expert $j$ is present, if it holds that $\pi_{E_j}(a|s) = \pi_{E_i} (a|s) = 1$, then we should have $\xi_i = \xi_j$ to avoid obtaining an empty reward set.} It follows that, is for each state-action pair $(s,a)$ such that $\pi_{E_1}(a|s) = 0$, there exists a single sub-optimal expert $i$ such that $\pi_{E_i}(a|s) = 1$, we are able to recover entirely a unique vector $\bar{\zeta}$ that is compatible with the underlying IRL problem. In other words, in this scenario, we are able to exactly recover the values of the advantage function that express the sub-optimality gaps of actions that are not played by the optimal agent.

\subsection{Measuring Volumes of the Feasible Values of $\zeta$}
Finally, we conclude this section by reporting an additional analysis that quantitatively measure the reduction in the feasible values of $\zeta$ that are compatible with the presence of the sub-optimal experts. 

To begin, we first introduce some notation. Let $\mathcal{X}$ be a measurable subset of $\mathbb{R}^{n}$, we denote with $\textrm{Vol}(\mathcal{X})$ the Lebesgue measure of $\mathcal{X}$ \citep{doob2012measure}. In other words, $\textrm{Vol}(\mathcal{X})$ represents the $n$-dimensional volume of $\mathcal{X}$.

At this point, from Theorem \ref{theorem:multiple-expert-explicit}, we know that the presence of sub-optimal experts can effectively limit the values that $\zeta$ can assume. Consequently, in order to measure the reduction of the feasible reward set, we will compute upper bounds on the volume of the region of $\zeta$ that induces at least one feasible reward function in $\mathcal{R}_{\bar{\mathfrak{B}}}$. In the remainder of this section, for a generic IRL problem $\mathfrak{B}$, we will denote with $Z(\mathfrak{B})$ such set. More specifically, we define:
\begin{align*}
Z(\mathfrak{B}) = \left\{ \zeta \in \mathbb{R}^{S \times A}_{\ge 0}: \exists r \in \mathcal{R}_{\mathfrak{B}}: \exists  V \in \mathbb{R}^{S} : r = -\bar{B}^{\pi_{E_1}} \zeta + (E-\gamma P)V \textup{ and } \zeta(s,a) = 0 \textup{ } \forall (s,a): \pi_{E_1}(a|s) = 0\right\}.
\end{align*}
Notice that, we are directly restricting the analysis to state-action pairs for which $\pi_{E_1}(a|s) = 0$ holds. Indeed, fix $(s,a)$ such that $\pi_{E_1}(a|s) = 0$ holds. Then, changing the value of $\zeta(s,a)$ does not affect class of compatible reward functions.

At this point, the following proposition provides upper bounds on the volume of the feasible values of $\zeta$ for IRL and IRL-SE problems.
\begin{proposition}\label{prop:volumes}
Let $\mathfrak{B}$ and $\bar{\mathfrak{B}}$ be an IRL and an IRL-SE problem. Let $g(s,a)$ be defined as follows:
\begin{align*}
g(s,a) \coloneqq \begin{cases}
						\min\left\{ k(s,a) , \frac{1}{1-\gamma}\right\} & \quad \textup{if } (s,a): \pi_{E_1}(a|s) = 0 \textup{ and } \Big|\mathcal{X}(s,a)\Big| > 0 \\ 
						\frac{1}{1-\gamma} &  \quad \textup{otherwise}
				  \end{cases},
\end{align*}
where $k(s,a)$ is defined as in Equation \eqref{eq:ksa}, and $\mathcal{X}(s,a)$ denotes the subset of sub-optimal experts such that $\pi_{E_i}(a|s) > 0$. Then, it holds that:
\begin{align}
& \textup{Vol}(Z(\mathfrak{B})) \le \prod_{(s,a): \pi_{E_1}(s,a) = 0} \frac{1}{1-\gamma}, \label{eq:vol1} \\ 
& \textup{Vol}(Z(\bar{\mathfrak{B}})) \le \prod_{(s,a): \pi_{E_1}(s,a) = 0} g(s,a) \label{eq:vol2}. 
\end{align}
\end{proposition}
\begin{proof}
The proof of Equation \eqref{eq:vol1} follows directly by noticing that $||\zeta||_\infty \le (1-\gamma)^{-1}$. Therefore, as a consequence:
\begin{align*}
\textup{Vol}(Z(\mathfrak{B})) \le \prod_{(s,a): \pi_{E_1}(a|s) = 0} \int_{0}^{(1-\gamma)^{-1}} 1 dx = \prod_{(s,a): \pi_{E_1}(s,a) = 0} \frac{1}{1-\gamma}.
\end{align*}

Equation \eqref{eq:vol2}, instead, follows from worst-case upper bounds on $\zeta(s,a)$ that arise from Equation \eqref{eq:explicit-eq2}. Specifically, as discussed in Section \ref{sec:explicit}, it is possible to show that:
\begin{align*}
\zeta(s,a) \le g(s,a).
\end{align*}
As a consequence, we have that:
\begin{align*}
\textup{Vol}(Z(\bar{\mathfrak{B}})) \le \prod_{(s,a): \pi_{E_1}(a|s) = 0} \int_{0}^{g(s,a)} 1 dx = \prod_{(s,a): \pi_{E_1}(s,a) = 0} g(s,a),
\end{align*}
which concludes the proof.
\end{proof}

At this point, we recall that $k(s,a)$ is given by:
\begin{align*}
k(s,a) \coloneqq \min_{i \in \mathcal{X}(s,a), s' \in \mathcal{S}} \frac{\xi_i}{d^{\pi_{E_i}}_{s'}(s)\pi_{E_i}(a|s)}.
\end{align*}
Therefore, Proposition \ref{prop:volumes} highlights that, for sufficiently small values of $\xi_i$, we obtain a notable reduction in the upper bounds of $\textup{Vol}(Z(\bar{\mathfrak{B}}))$.

\section{Proofs and Derivations of Section \ref{sec:learning}}\label{app:proof-app-sec-4}

In this section, we provide formal proofs for the statements of Section \ref{sec:learning}. We first proceed with the proof of the lower bound (Theorem \ref{theo:lb}), and then we continue with the analysis US-IRL-SE (Theorem \ref{theo:ub}).

\subsection{Proof of Theorem \ref{theo:lb}}
In this section, we provide a formal proof for Theorem \ref{theo:lb}.

As we have discussed, Theorem \ref{theo:lb} is composed of two parts. The first one is related to learning reward functions that are compatible with the policy of the optimal expert (i.e., Equation \eqref{eq:lb-one}), while the second one directly arises from the structure of the sub-optimal experts (i.e., Equation \eqref{eq:lb-two}). 

We begin by proving Equation \eqref{eq:lb-one}. As we have anticipated in Section \ref{sec:lower-bound}, Equation \eqref{eq:lb-one} is directly connected with the problem of learning reward functions that are compatible with the behavior of the optimal expert $\pi_{E_1}$. In this sense, we recall that, recently, \citet{metelli2023towards} have provided lower bounds for the single-agent IRL problem in finite-horizon MDPs. In our work, we provide an extension of these results for the infinite-horizon IRL formulation. 

\begin{theorem}\label{theo:lb-opt-exp}
Let $\mathfrak{A}$ be a $(\epsilon, \delta)$-correct algorithm for the IRL problem with sub-optimal experts. There exists a problem instance $\bar{\mathfrak{B}}$ such that the expected sample complexity is lower bounded by:
\begin{align}
	\mathbb{E}_{\mathfrak{A}, \bar{\mathfrak{B}}}[\tau] \ge {\Omega} \left(\frac{SA}{\epsilon^2 (1-\gamma)^2} \left(\log\left( \frac{1}{\delta}\right) + S \right) \right).
\end{align}
\end{theorem}
\begin{proof}
Similarly to \citet{metelli2023towards}, this results follows from two different lower bounds (i.e., Theorem \ref{theo:lb-opt-exp-caso1} and Theorem \ref{theo:lb-opt-exp-caso2}), and by assuming to observe instances like the ones of Theorem \ref{theo:lb-opt-exp-caso1} with probability $\frac{1}{2}$ and instances like the ones of Theorem \ref{theo:lb-opt-exp-caso2} with probability $\frac{1}{2}$.
\end{proof}

As discussed, Theorem \ref{theo:lb-opt-exp} follows from two intermediate results that leverage two different constructions (i.e., Theorem \ref{theo:lb-opt-exp-caso1} and Theorem \ref{theo:lb-opt-exp-caso2}). We now delve into the proofs of these intermediate theorems.

\begin{figure}[t]
\centering\includegraphics[width=15cm]{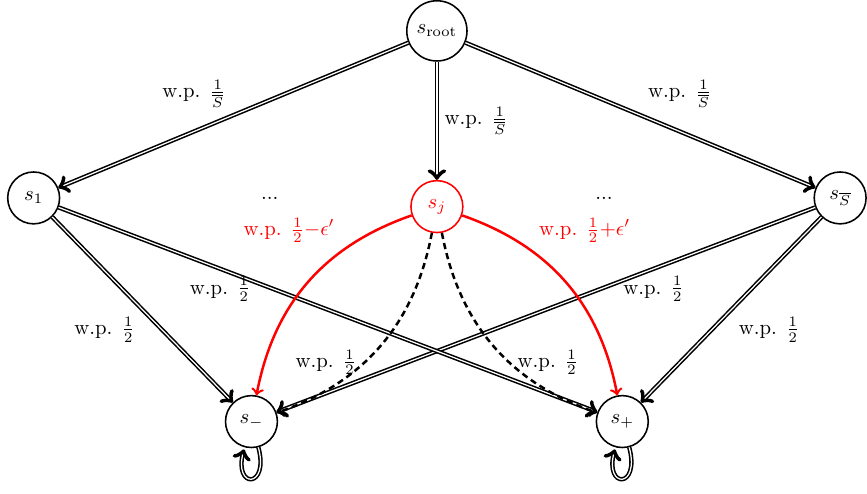} 
\vspace{.3in}
\caption{Representation of the IRL-SE problem for the instances used in Theorem \ref{theo:lb-opt-exp-caso1}.}
\label{fig:instance-opt-agent-uno}
\end{figure}

\begin{theorem}\label{theo:lb-opt-exp-caso1}
Let $\mathfrak{A}$ be a $(\epsilon, \delta)$-correct algorithm for the IRL problem with sub-optimal experts. There exists a problem instance $\bar{\mathfrak{B}}$ such that the expected sample complexity is lower bounded by:
\begin{align}
	\mathbb{E}_{\mathfrak{A}, \bar{\mathfrak{B}}}[\tau] \ge {\Omega} \left(\frac{SA}{\epsilon^2 (1-\gamma)^2} \log\left( \frac{1}{\delta}\right) \right).
\end{align}
\end{theorem}
\begin{proof}
The proof is split into several steps. Since most of the arguments are borrowed from the proof of Theorem B.2 in \citet{metelli2023towards}, we only provide a short sketch that involves the main differences between the finite-horizon and the infinite horizon IRL model.

\paragraph{Step 1: Base Instance and Alternative Instances} 
We consider MDP\textbackslash R instances that are presented in Figure~\ref{fig:instance-opt-agent-uno}. More specifically, the state space is given by $\mathcal{S} = \left\{ s_{\textup{root}}, s_1, \dots, s_{\bar{S}}, s_{-}, s_{+} \right\}$, the action space is composed of $k$ actions $\mathcal{A} = \left\{ a_1, a_2, \dots, a_k \right\}$, the transition model is described below, and $\gamma \in (0, 1)$. 

In state $s_{\textup{root}}$ all actions have the same effect, and they lead, with probability $\frac{1}{\bar{S}}$ to a state in $\left\{ s_1, \dots, s_{\bar{S}} \right\}$. Similarly, in state $s_{-}$ and $s_{+}$ all actions have deterministic effect, and they all lead to $s_{-}$ and $s_{+}$ respectively. All the experts are deterministic and take action $a_1$ with probability $1$ in all states. We then consider a base instance, where, in each state $s_j \in \left\{s_1, \dots, s_{\bar{S}} \right\}$, all actions lead, with probability $\frac{1}{2}$ to $s_{-}$ and $s_{+}$.

We then consider a set of $\bar{S} \times A$ alternative instances by varying the behavior of the transition model in state-action pairs $(s_j, a_k) \in \left\{ s_1, \dots s_{\bar{S}}\right\} \times \mathcal{A}$. Specifically, by taking $a_k$, the agent will end up, with probability $\frac{1}{2} +\epsilon'$ in $s_+$, and, with probability $\frac{1}{2} - \epsilon'$ in $s_{-}$.

\paragraph{Step 2: Feasible Reward Set and Hausdorff Distance lower bound}
At this point, we study the structure of the feasible reward set that is compatible with the instances we described. Specifically, we are interested in the behavior of the reward function related to actions taken in states $s_j \in \left\{ s_1, \dots s_{\bar{S}}\right\}$. 

Specifically, for the base instance, we have:
\begin{align*}
r(s_j,a_1) + \frac{\gamma}{1-\gamma} \left[ \frac{1}{2}r(s_-, a_1) + \frac{1}{2}r(s_+, a_1)\right] \ge r(s_j,a_k) + \frac{\gamma}{1-\gamma} \left[ \frac{1}{2}r(s_-, a_1) + \frac{1}{2}r(s_+, a_1)\right],
\end{align*}
which can be rewritten as:
\begin{align*}
r(s_j,a_1) \ge r(s_j,a_k).
\end{align*}
For the alternative instance in which we varied the behavior of the state-action pair $(s_j, a_k)$, instead, we obtain:
\begin{align*}
r(s_j,a_1) + \frac{\gamma}{1-\gamma} \left[ \frac{1}{2}r(s_-, a_1) + \frac{1}{2}r(s_+, a_1)\right] \ge r(s_j,a_k) + \frac{\gamma}{1-\gamma} \left[ \left(\frac{1}{2}-\epsilon'\right)r(s_-, a_1) + \left(\frac{1}{2}+\epsilon'\right) r(s_+, a_1)\right],
\end{align*}
thus leading to:
\begin{align*}
r(s_j,a_1) \ge r(s_j,a_k) - \epsilon' \frac{\gamma}{1-\gamma} \left[ r({s_{-},a_1}) - r(s_+, a_1) \right],
\end{align*}
Given this construction, we can lower bound the Hausdorff distance between the feasible reward set of the base and the alternative instance. Specifically, we first pick a reward function $r'$ compatible with the alternative instance as follows: $r'(s_-,a_1)=1, r'(s_+, a_1) = 0, r'(s_j, a_k) = 1$ and $r'(s_j, a_1) = 1 - \frac{\epsilon' \gamma}{1-\gamma}$. Then, we study which reward function compatible with the base instance minimizes the infinite norm w.r.t. $r'$. Given these choices, similarly to \citet{metelli2023towards}, it is possible to obtain the following lower-bound to the Hausdorff distance:
\begin{align*}
H_\infty(\mathcal{R}, \mathcal{R}') \ge \frac{\epsilon' \gamma}{1-\gamma},
\end{align*}
where $\mathcal{R}$ denotes the feasible reward set of the base instance, while $\mathcal{R}'$ denotes the feasible reward set of the alternative instance. We enforce the following constraint on the previous equation:
\begin{align*}
\frac{\epsilon' \gamma}{1-\gamma} \ge 2 \epsilon,
\end{align*}
which, in turns, leads to the following requirement on $\epsilon'$:
\begin{align*}
\epsilon' \ge \frac{2 \epsilon (1-\gamma)}{\gamma}.
\end{align*}

\paragraph{Step 3: Lower bounding the sample complexity} At this point, the rest of the proof follows identical to Theorem B.2 in \citet{metelli2023towards}, and leads to the desired result.

\end{proof}

We now continue with the proof of second intermidiate result that is needed for the proof of Theorem \ref{theo:lb-opt-exp}.

\begin{figure}[t]
\centering\includegraphics[width=17cm]{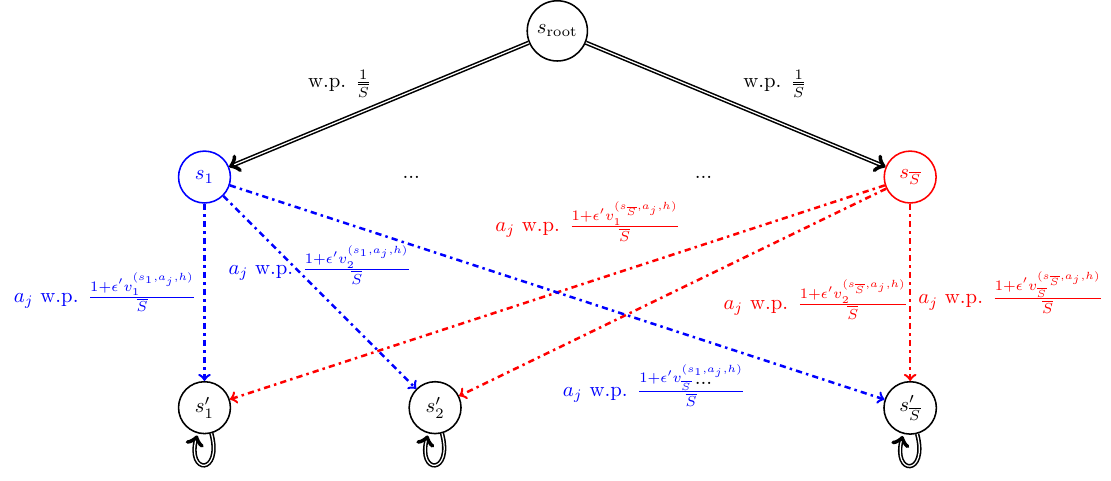} 
\vspace{.3in}
\caption{Representation of the IRL-SE problem for the instances used in Theorem \ref{theo:lb-opt-exp-caso2}.}
\label{fig:instance-opt-agent-due}
\end{figure}

\begin{theorem}\label{theo:lb-opt-exp-caso2}
Let $\mathfrak{A}$ be a $(\epsilon, \delta)$-correct algorithm for the IRL problem with sub-optimal experts. There exists a problem instance $\bar{\mathfrak{B}}$ such that the expected sample complexity is lower bounded by:
\begin{align}
	\mathbb{E}_{\mathfrak{A}, \bar{\mathfrak{B}}}[\tau] \ge {\Omega} \left(\frac{S^2A}{\epsilon^2 (1-\gamma)^2} \right).
\end{align}
\end{theorem}

\begin{proof}
The proof is split into several steps. Since most of the arguments are borrowed from the proof of Theorem B.3 in \citet{metelli2023towards}, we only provide a short sketch that involves the main differences between the finite-horizon and the infinite horizon IRL model.

\paragraph{Step 1: Base Instance and Alternative Instances}
We consider MDP\textbackslash R instances that are presented in Figure~\ref{fig:instance-opt-agent-due}. More specifically, the state space is given by $\mathcal{S} = \left\{ s_{\textup{root}}, s_1, \dots, s_{\bar{S}}, s'_1, \dots, s'_{\bar{S}}\right\}$, the action space is composed of $k$ actions $\mathcal{A} = \left\{ a_1, a_2, \dots, a_k \right\}$, the transition model is described below, and $\gamma \in (0, 1)$. In the following, we will assume $\bar{S}$ to be divisible by $16$.

In state $s_{\textup{root}}$ all actions have the same effect, and they lead, with probability $\frac{1}{\bar{S}}$ to a state in $\left\{ s_1, \dots, s_{\bar{S}} \right\}$. In state $s \in \left\{s'_1, \dots, s'_{\bar{S}} \right\}$ all actions have deterministic effect, and they all lead to the same state in which the action is taken. All the experts are deterministic and take action $a_1$ with probability $1$ in all states. Finally, in all states $s \in \left\{s_1, \dots s_{\bar{S}} \right\}$, action $a_1$ leads w.p. $\frac{1}{{\bar{S}}}$ to any state in $s'_{\bar{S}} \in \left\{s'_1, \dots, s'_{\bar{S}} \right\}$.

We then consider a class of instances by varying the behavior of the transition model in state-action pairs $(s_j, a_k) \in \left\{ s_1, \dots s_{\bar{S}}\right\} \times \mathcal{A} \setminus \left\{ a_1 \right\}$. Specifically, for each triplet $(s_j, s'_i, a_k)$, $p(s'_{i}|s_j, a_k ) = \frac{1+\epsilon' v_i^{(s_j, a_k)}}{2}$, where $v^{(s_j,_k)} =  \left(v_1^{(s_i, a_j)}, \dots v_{\bar{S}}^{(s_i, a_j)} \right) \in \mathcal{V}$, and $\mathcal{V} = \left\{ \left\{ -1, 1\right\}^{\bar{S}}: \sum_{i}^{\bar{S}} v_j = 0 \right\}$ and $\epsilon' \in \left[0, \frac{1}{2}\right]$. 


\paragraph{Step 2: Feasible Reward Set and Hausdorff Distance lower bound }
We now proceed with lower bounding the Hausdorff distance between the feasible reward set that is induced by two instances that belongs to $\mathcal{V}$, namely $\mathcal{R}_{\bar{\mathfrak{B}}_v}$ and $\mathcal{R}_{\bar{\mathfrak{B}}_w}$. To this end, first of all, we notice that our class of reward functions admits elements that are bounded in $[0,1]$. Let us denote with $\bar{\mathcal{R}}_{\bar{\mathfrak{B}}_v}$ the class of compatible rewards functions that are bounded in $[-1,1]$. Then, we have that:
\begin{align}\label{eq:rescale-trick}
H_\infty\left(\mathcal{R}_{\bar{\mathfrak{B}}_v}, \mathcal{R}_{\bar{\mathfrak{B}}_w}\right) \ge \frac{1}{2}H_\infty \left( \bar{\mathcal{R}}_{\bar{\mathfrak{B}}_v}, \bar{\mathcal{R}}_{\bar{\mathfrak{B}}_w}\right).
\end{align}
The proof of Equation \eqref{eq:rescale-trick} follows by the following reasonings. First of all, it is easy to see that, for any $v$, $\bar{r} \in \bar{\mathcal{R}}_{\bar{\mathfrak{B}}_v}$ holds if and only if $\frac{\bar{r}+1}{2} \in \mathcal{R}_{\bar{\mathfrak{B}}_v}$ holds.\footnote{Notice that we are considering single-experts IRL problems. Indeed, in the considered IRL-SE problem, all sub-optimal experts behave identically to the optimal agent, and the constraints of Equation \eqref{eq:explicit-eq2} are vacuous.} Then, with simple algebraic manipulations, we have that:
\begin{align*}
\sup_{x \in \mathcal{R}_{\bar{\mathfrak{B}}_v}} \inf_{y \in \mathcal{R}_{\bar{\mathfrak{B}}_w}} ||x-y||_\infty &  = \sup_{x \in \bar{\mathcal{R}}_{\bar{\mathfrak{B}}_v}} \inf_{y \in \bar{\mathcal{R}}_{\bar{\mathfrak{B}}_w}} \Big|\Big|\frac{x+1}{2}-\frac{y+1}{2}\Big|\Big|_\infty \\ & = \frac{1}{2} \sup_{x \in \bar{\mathcal{R}}_{\bar{\mathfrak{B}}_v}} \inf_{y \in \bar{\mathcal{R}}_{\bar{\mathfrak{B}}_w}} || x - y ||_\infty,
\end{align*}
from which it follows Equation \eqref{eq:rescale-trick}.

At this point, our analysis follows by lower bounding the Hausdorff distance using reward functions that are bounded in $[-1,1]$. To this end, first of all, we analyze the feasible reward set for a single instance $\bar{\mathfrak{B}}_v$. In this case, since $a_1$ is played by the optimal expert, we have that:
\begin{align*}
r(s_j, a_1) + \frac{1}{\bar{S}} \frac{\gamma}{1-\gamma} \sum_{s'_i} r^v(s'_i) \ge r(s_j, a_k)  + \frac{1}{\bar{S}} \frac{\gamma}{1-\gamma} \sum_{s'_i} \left(1+\epsilon' v_i\right) r^v(s'_i),
\end{align*}
which, in turns, leads to:
\begin{align*}
r(s_j, a_1) \ge r(s_j, a_k)  + \epsilon' \frac{1}{\bar{S}} \frac{\gamma}{1-\gamma} \sum_{s'_i} v_i r^v(s'_i).
\end{align*}
At this point, following the same steps of \citet{metelli2023towards}, we obtain:
\begin{align*}H_\infty\left(\mathcal{R}_{\bar{\mathfrak{B}}_v}, \mathcal{R}_{\bar{\mathfrak{B}}_w}\right) \ge \frac{1}{2}H_\infty \left( \bar{\mathcal{R}}_{\bar{\mathfrak{B}}_v}, \bar{\mathcal{R}}_{\bar{\mathfrak{B}}_w}\right) \ge \frac{\epsilon'}{64} \frac{\gamma}{1-\gamma},
\end{align*}
on which we enforce the following constraint:
\begin{align*}
\frac{\epsilon'}{64} \frac{\gamma}{1-\gamma} \ge 2 \epsilon,
\end{align*}
that is $\epsilon' \ge 128 \epsilon \frac{1-\gamma}{\gamma}$.

\paragraph{Step 3: Lower-bounding the sample complexity}
At this point, the rest of the proof follows identical to Theorem B.3 in \citet{metelli2023towards}, and leads to the desired result.

\end{proof}

\begin{figure}[t]
\centering\includegraphics[width=15cm]{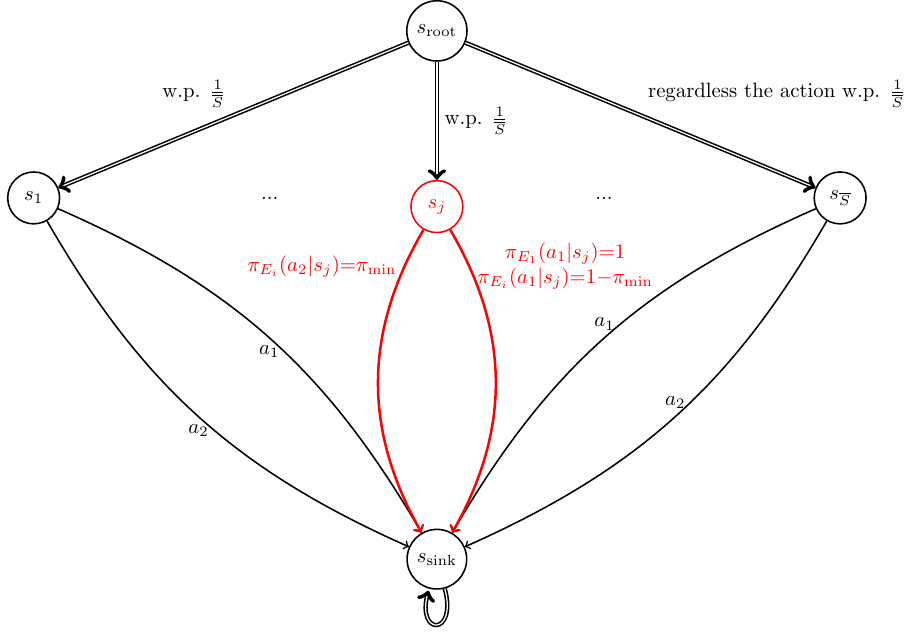} 
\vspace{.3in}
\caption{Representation of the IRL-SE problem for the instances used in Theorem \ref{theo:lb-sub-opt}.}
\label{fig:instance-sub-opt}
\end{figure}

We continue by proving the novel statement related to Equation \eqref{eq:lb-two}, which captures the role of the estimating reward functions that are compatible with the behavior of the sub-optimal experts. 

\begin{theorem}\label{theo:lb-sub-opt}
Let $\mathfrak{A}$ be a $(\epsilon, \delta)$-correct algorithm for the IRL problem with sub-optimal experts. For sufficiently small values of $\epsilon$, there exists an instance $\bar{\mathfrak{B}}$ in which $q_0 < 1$ such that the expected sample complexity is lower bounded by:
\begin{align}\label{eq:lb-two}
	\mathbb{E}_{\mathfrak{A}, \bar{\mathfrak{B}}}[\tau] \ge {\Omega} \left( \frac{q_0^2 S \log\left( \frac{1}{\delta} \right)}{\epsilon^2 \pi_\textup{min}} \right).
\end{align}
\end{theorem}
\begin{proof}
The proof follows as the combination of several steps. We first describe the construction considering the case in which there is a single sub-optimal expert. Then, the results will follow by iterating the procedure for each sub-optimal expert in the given set. 

\paragraph{Step 1: Base Instance and Alternative Instance}
We consider MDP\textbackslash R instances that are presented in Figure~\ref{fig:instance-sub-opt}. More specifically, the state space is given by $\mathcal{S} = \left\{ s_{\textup{root}}, s_1, \dots, s_{\bar{S}}, s_{\textup{sink}} \right\}$, the action space is composed of two actions $\mathcal{A} = \left\{ a_1, a_2 \right\}$, the transition model is described below, and $\gamma \in (0, 1)$. 

In state $s_{\textup{root}}$ all actions have the same effect, and they lead, with probability $\frac{1}{\bar{S}}$ to a state in $\left\{ s_1, \dots, s_{\bar{S}} \right\}$. In each state $s_i \in \left\{ s_1, \dots, s_{\bar{S}}\right\}$, again, all actions have the same effect, and they lead, deterministically to state $s_{\textup{sink}}$. In state $s_{\textup{sink}}$ all actions are deterministic and they all lead to $s_{\textup{sink}}$.

We then consider a base instance, where, in states $s_{\textup{root}}$ and state $s_{\textup{sink}}$ all experts take action $a_1$ with probability $1$. In each state $s_j \in \left\{s_1, \dots, s_{\bar{S}} \right\}$, instead, $\pi_{E_1}(a_1|s_j) = 1$, while $\pi_{E_i}(a_1|s_j) = 1-\pi_{\textup{min}}$ and $\pi_{E_i}(a_2|s_j) = \pi_{\textup{min}}$ for the sub-optimal expert.

We then consider a set of $\bar{S}$ alternative instances by varying the behavior of the sub-optimal expert in state $s_j \in \left\{ s_1, \dots s_{\bar{S}}\right\}$. Specifically, the sub-optimal expert will act accordingly to the following policy: $\pi_{E_i}(a_1|s_j) = 1-\alpha \pi_{\textup{min}}$, and $\pi_{E_i}(a_2|s_j) = \alpha \pi_\textup{min}$ for some $\alpha > 1$ (to be defined later). We will denote our set of instances as $\mathbb{M} = \left\{ \pi_{E_i}^{(l)}: l \in [\bar{S}] \right\}$, and $0$ denotes the behavior of the base instance we described above.

\paragraph{Step 2: Feasible Reward Set}
At this point, we study the structure of the feasible reward set that is compatible with the each instance we described. Specifically, we are interested in the properties of feasible reward functions in states $s_j \in \left\{ s_1, \dots, s_{\bar{S}}\right\}$. 

For the base instance, in each state $s_j \in \left\{ s_1, \dots, s_{\bar{S}}\right\}$, we have that, since $a_1$ is the action taken by the optimal expert:
\begin{align*}
r(s_j, a_1) + \frac{\gamma}{1-\gamma} r(s_{\textup{sink}}, a_1) \ge r(s_j,a_2) + \frac{\gamma}{1-\gamma} r(s_{\textup{sink}}, a_1),
\end{align*} 
which reduces to:
\begin{align}\label{eq:lb-proof-ours-1}
r(s_j,a_1) \ge r(s_j,a_2).
\end{align} 
Moreover, due to the presence of the sub-optimal expert, we have that $V^{\pi_{E_1}^{(0)}}_{\mathcal{M} \cup r}(s_j) \le V^{\pi_{E_i}}_{\mathcal{M} \cup r}(s_j) + \xi_i$, which results in:
\begin{equation*}
\resizebox{\linewidth}{!}{$\displaystyle
r(s_j,a_1) + \frac{\gamma}{1-\gamma} r(s_{\textup{sink}},a_1) \le \left( 1-\pi_\textup{min} \right) \left( r(s_j,a_1) + \frac{\gamma}{1-\gamma} r(s_{\textup{sink}},a_1) \right) + \pi_{\textup{min}} \left( r(s_j,a_2) + \frac{\gamma}{1-\gamma} r(s_{\textup{sink}},a_1) \right) + \xi_i,$
}%
\end{equation*}
which can be rewritten as:
\begin{align}\label{eq:lb-proof-ours-2}
r(s_j,a_1) - r(s_j,a_2) \le \frac{\xi_i}{\pi_{\textup{min}}}.
\end{align}

Similarly, for each alternative instance in $\mathbb{M}$, we obtain:
\begin{align}
& r(s_j,a_1) \ge r(s_j,a_2) \\
& r(a_1,s_j) - r(s_j,a_2) \le \frac{\xi_i}{\alpha \pi_{\textup{min}}} \label{eq:lb-proof-ours-3}.
\end{align}

\paragraph{Step 3: Lower-bounding the Hausdorff Distance}
At this point, we proceed by lower bounding the Hausdorff distance between the feasible reward set of the base instance, namely $\mathcal{R}_{\mathfrak{B}^{(0)}}$, and that of alternative instance $j$ in which the policy of the sub-optimal expert is modified in state $s_j$, namely $\mathcal{R}_{\mathfrak{B}^{(j)}}$. 

To this end, we notice that it is sufficient to pick any $r \in \mathcal{R}_{\mathfrak{B}^{(0)}}$, and then study the following optimization problem:
\begin{align}\label{eq:lb-proof-ours-4}
\inf_{r' \in \mathcal{R}_{\mathfrak{B}^{(j)}}} ||r-r' ||_\infty.
\end{align}
We pick the following choice for $r$. For action $a_1$ in state $s_j$ we pick $r(s_j,a_1) = \frac{\xi_i}{\pi_{\textup{min}}}$. Instead, for all the other state-action pair, we pick $r(\cdot, \cdot) = 0$.\footnote{Notice that in this step we have used the assumption that $q_0 < 1$.} In this case, it is easy to verify that the optimization problem in \eqref{eq:lb-proof-ours-4} can be rewritten as follows:

\begin{equation}\label{eq:lb-proof-ours-5}
\begin{aligned} 
\inf_{x,y} \quad & \max\left\{ \Big| \frac{\xi_i}{\pi_{\textup{min}}} - x \Big|,  y  \right\}  \\
\textrm{s.t.} \quad &  x,y \in [0,1] \\
  & x-y \le \frac{\xi_i}{\alpha \pi_{\textup{min}}} \\
  & x \ge y,
\end{aligned}
\end{equation}
where $x$ corresponds to $r'(s_j,a_1)$, and $y$ to $r'(s_j,a_2)$. The optimization problem in \eqref{eq:lb-proof-ours-5} can be further lower-bounded by:
\begin{equation}\label{eq:lb-proof-ours-6}
\begin{aligned} 
\inf_{x,y} \quad & \Big| \frac{\xi_i}{\pi_{\textup{min}}} - x \Big|\\
\textrm{s.t.} \quad &  x,y \in [0,1] \\
  & x-y \le \frac{\xi_i}{\alpha \pi_{\textup{min}}} \\
  & x \ge y \\
  & \Big| \frac{\xi_i}{\pi_{\textup{min}}} - x \Big| \ge y  
\end{aligned}
\end{equation}
We study the optimization problem in \eqref{eq:lb-proof-ours-6} by cases (i.e., ${\xi_i}{\pi_{\textup{min}}}^{-1} < x$, and ${\xi_i}{\pi_{\textup{min}}}^{-1} \ge x$). First, we suppose that ${\xi_i}{\pi_{\textup{min}}}^{-1} < x$. However, in this case, we would have $y \le x - {\xi_i}{\pi_{\textup{min}}}^{-1}$, that is $x \ge y + {\xi_i}{\pi_{\textup{min}}}^{-1}$. Chaining this inequality with $x \le y + {\xi_i}\left({\alpha \pi_{\textup{min}}}\right)^{-1}$ yields:
\begin{align*}
y + \frac{\xi_i}{\pi_{\textup{min}}} \le x \le y + \frac{\xi_i}{\alpha \pi_{\textup{min}}},
\end{align*}
which is impossible since $\alpha > 1$. Concerning the case in which ${\xi_i}{\pi_{\textup{min}}}^{-1} \ge x$, instead, we have that $y \le {\xi_i}{\pi_{\textup{min}}}^{-1} - x$. It follows that, since $x \le \xi_i \pi_{\textup{min}}^{-1} + y$, the maximum value of $x$ will be attained for the maximum value of $y$. Therefore, we obtain:
\begin{align*}
x \le \frac{\xi_i}{\alpha \pi_{\textup{min}}} + \frac{\xi_i}{\pi_{\textup{min}}} - x,
\end{align*}
Thus leading to:
\begin{align*}
x \le \frac{1}{2} \frac{\xi_i}{\pi_{\textup{min}}} \left(\frac{1}{\alpha} + 1 \right).
\end{align*}
Plugging this result into the objective function, we obtain the solution to optimization problem \eqref{eq:lb-proof-ours-6}, which in turn provides a lower bound to the Hausdorff distance. More specifically, we have that:
\begin{align}\label{eq:lb-proof-ours-7}
H_\infty \left( \mathcal{R}_{\mathfrak{B}^{(0)}}, \mathcal{R}_{\mathfrak{B}^{(j)}} \right) \ge \frac{1}{2} \frac{\xi_i}{\pi_{\textup{min}}} \left( 1 - \frac{1}{\alpha} \right).
\end{align}

We enforce the following constraint on this quantity, that is:
\begin{align}\label{eq:lb-proof-ours-8}
\frac{1}{2} \frac{\xi_i}{\pi_{\textup{min}}} \left( 1 - \frac{1}{\alpha} \right) = 2 \epsilon.
\end{align}

We notice that in this sense, we need to impose $\alpha = {\xi_i} \left(\xi_i - 4\epsilon\pi_{\textup{min}}\right)^{-1}$, which results being greater than $1$ for all values of $\xi_i > 4 \epsilon \pi_{\textup{min}}$. Furthermore, since $\alpha \pi_{\textup{min}} < 1$,\footnote{This condition directly arises from the fact that $\alpha \pi_{\textup{min}}$ represents the probability with which the sub-optimal expert takes a certain action.} we obtain the following condition on the values of $\epsilon$ which guarantes that the construction is valid:
\begin{align*}
\epsilon < \min \left\{ \frac{\xi_i}{4\pi_{\textup{min}}}, \frac{1-\xi_i}{4} \right\}.
\end{align*}

\paragraph{Step 4: Lower-bounding Probability} At this point, consider an $\left(\epsilon, \delta\right)$-correct algorithm $\mathfrak{A}$ that outputs the correct feasible set $\mathcal{R}_{\widehat{\bar{\mathfrak{B}}}}$ for each IRL-SE problem $\bar{\mathfrak{B}}$. For $\mathfrak{A}$ it holds that:

\begin{align}
\delta & \ge \sup_{\bar{\mathfrak{B}}} \mathbb{P}_{\bar{\mathfrak{B}},\mathfrak{A}} \left( H\left( \mathcal{R}_{{\bar{\mathfrak{B}}}},\mathcal{R}_{\widehat{\bar{\mathfrak{B}}}} \right) \ge \epsilon \right) \\ & \ge \max_{\mathfrak{B} \in \mathbb{M}} \mathbb{P}_{\bar{\mathfrak{B}},\mathfrak{A}} \left( H\left( \mathcal{R}_{{\bar{\mathfrak{B}}}},\mathcal{R}_{\widehat{\bar{\mathfrak{B}}}} \right) \ge \epsilon \right),
\end{align}
where $\mathbb{P}_{\bar{\mathfrak{B}}, \mathfrak{A}}$ denotes the joint probability distribution of all events realized by the execution of algorithm $\mathfrak{A}$ in the IRL-SE problem $\bar{\mathfrak{B}}$, that is:
\begin{align*}
\mathbb{P}_{\bar{\mathfrak{B}}, \mathfrak{A}} = \prod_{t=1}^{\tau} \nu_t\left(s_t, a_t|\mathcal{D}_{t-1}\right) p(s'_t|s_t,a_t) \prod_{i=1}^{n+1} \pi_{E_i}(a_t^{(i)}|s_t)
\end{align*}

At this point, following the same reasonings of, e.g., Theorem B.2 in \citet{metelli2023towards}, it is easy to see that, for all the alternative instances in $l \in \mathbb{M}$, the previous equation can be further lower-bounded to obtain:
\begin{align}\label{eq:lb-proof-ours-9}
\delta \ge \frac{1}{4} \exp\left( -\textup{KL}\left( \mathbb{P}_{\bar{\mathfrak{B}}^{(0)},\mathfrak{A}}, \mathbb{P}_{\bar{\mathfrak{B}}^{(l)},\mathfrak{A}} \right) \right),
\end{align}

\paragraph{Step 5: KL-divergence} As a final step, we will proceed by analyzing Equation \eqref{eq:lb-proof-ours-9}.
We will start by analyzing the KL divergence between $\mathbb{P}_{\bar{\mathfrak{B}}^{(0)},\mathfrak{A}}$ and $\mathbb{P}_{\bar{\mathfrak{B}}^{(0)},\mathfrak{A}}$. Specifically, we have that:
\begin{align*}
\textup{KL}\left( \mathbb{P}_{\bar{\mathfrak{B}}^{(0)},\mathfrak{A}}, \mathbb{P}_{\bar{\mathfrak{B}}^{(l)},\mathfrak{A}} \right) & = \mathbb{E}_{\bar{\mathfrak{B}}^{(0)},\mathfrak{A}} \left[ \sum_{t=1}^\tau \textup{KL} \left(  \pi_{E_i}^{(0)}(\cdot|s_t), \pi_{E_i}^{{(l)}}(\cdot|s_t) \right) \right] \\ & \le \mathbb{E}_{\bar{\mathfrak{B}}^{(0)},\mathfrak{A}} \left[ N_\tau(s_j) \right] \frac{\textup{TV}\left(\pi_{E_i}^{(0)}(\cdot|s_j), \pi_{E_i}^{(l)}(\cdot|s_j)\right)^2}{\pi_{\textup{min}}} \\ & \le \mathbb{E}_{\bar{\mathfrak{B}}^{(0)},\mathfrak{A}} \left[ N_\tau(s_j) \right] \frac{\left( \alpha \pi_{\textup{min}} - \pi_{\textup{min}} \right)^2}{\pi_{\textup{min}}},
\end{align*}
where $\textup{TV}(p_1,p_2)$ represents the total variation distance between distribution $p_1$ and $p_2$. The first inequality step follows from the reverse Pinkser's inequality \citep{sason2015reverse}, and the second one, instead, from the definition of the total variation distance.

At this point, however, from Equation \eqref{eq:lb-proof-ours-8}, we obtain that:
\begin{align*}
\alpha \pi_{\textup{min}} - \pi_{\textup{min}} = \frac{4 \alpha \pi_{\textup{min}}^2 \epsilon}{\xi_i},
\end{align*}
thus leading to:
\begin{align}
\textup{KL}\left( \mathbb{P}_{\bar{\mathfrak{B}}^{(0)},\mathfrak{A}}, \mathbb{P}_{\bar{\mathfrak{B}}^{(l)},\mathfrak{A}} \right) \le \mathbb{E}_{\bar{\mathfrak{B}}^{(0)},\mathfrak{A}} \left[ N_\tau(s_j) \right] \frac{16 \alpha^2 \pi_{\textup{min}}^3 \epsilon^2}{\xi_i^{2}}
\end{align}

Plugging this result into Equation \eqref{eq:lb-proof-ours-9}, and rearranging the terms, we obtain:
\begin{align}\label{eq:lb-proof-ours-9}
\mathbb{E}_{\bar{\mathfrak{B}}^{(0)},\mathfrak{A}} \left[ N_\tau(s_j) \right] \ge \frac{\log \left( \frac{1}{4\delta}\right) \xi_i^2}{16 \epsilon^2 \pi_{\textup{min}}^3 \alpha^2} \ge \frac{\xi_i^2 \log \left( \frac{1}{4\delta}\right) }{16 \epsilon^2 \pi_{\textup{min}}^3}.
\end{align}

Finally, to conclude, we notice that:
\begin{align}
\mathbb{E}_{\bar{\mathfrak{B}}^{(0)},\mathfrak{A}} \left[ \tau \right] & \ge \sum_{s \in \left\{s_1, \dots, s_{\bar{S}} \right\}} \mathbb{E}_{\bar{\mathfrak{B}}^{(0)},\mathfrak{A}} \left[ N_\tau(s) \right] \\ & \ge \frac{(S-2) \xi_i^2}{16 \epsilon^2 \pi_{\textup{min}}^3} \log \left( \frac{1}{4\delta}\right)
\end{align}

Iterating this procedure over all the possible experts yields the desired result.

\end{proof}

Finally, we are now ready to prove Theorem \ref{theo:lb}.

\lb*
\begin{proof}
The proof follows directly by combining Theorem \ref{theo:lb-opt-exp} and Theorem \ref{theo:lb-sub-opt}.
\end{proof}

\subsection{Proof of Theorem \ref{theo:ub}}
In this section, we provide a formal proof for Theorem \ref{theo:ub}.

First of all, we begin by defining a good event $\mathcal{E}$ that holds with probability at least $1-\delta$. To this end, we begin by reporting, for the sake of completeness, the concentration tools that are used in controlling the probability with which $\mathcal{E}$ holds.

\begin{lemma}[Multiplicative Chernoff bound \citep{hagerup1990guided}]\label{lemma:multiplicative-chern}
Consider $t$ independent random variables $X_1, \dots, X_t$ taking values in $[0, 1]$. Suppose that $\mathbb{E}[X_i] = \mu$ holds for all $i \in \left\{1, \dots, t \right\}$. Consider $\alpha \in (0, 1)$. Then, we have that:
\begin{align*}
\mathbb{P} \left( \sum_{i=1}^t X_i \ge t \mu (1+\alpha) \right) \le \exp \left( -\frac{1}{3}t \mu \alpha^2 \right).
\end{align*}
\end{lemma}

\begin{lemma}[Proposition 1 in \citet{jonsson2020planning}]\label{lemma:johnsonn}
Let $p$ be a categorical distribution over the simplex of dimension $y$, and let $\hat{p}$ be the maximum likelihood estimate of $p$ obtained with $t \ge 1$ independent samples. Then, for all $\delta \in (0,1)$ it holds that:
\begin{align*}
\mathbb{P} \left(\exists t \ge 1: t \textup{KL}\left(\hat{p}, p \right) > \log\left(\frac{1}{\delta} \right) + (y-1) \log\left( e\left(1+\frac{t}{y-1}\right) \right) \right) \le \delta,
\end{align*}
where $\textup{KL}(q_1,q_2)$ denotes the Kullback–Leibler divergence between distributions $q_1$ and $q_2$. 
\end{lemma}

At this point, we proceed by defining our good event. In the following, the subscript $t$ is used to denote the iteration number of the US-IRL-SE algorithm. 

\begin{lemma}[Good event]\label{lemma:good-event}
Consider $t$ such that:
\begin{align}\label{eq:good-event-ass}
\sqrt{\frac{3\log\left( \frac{3SAn}{\delta} \right)}{\pi_{E_i}(a|s) N_t(s)}}  < 1
\end{align}
holds for all $i \in \left\{2, \dots, n+1 \right\}$ and for all $(s,a) \in \mathcal{S} \times \mathcal{A}$ such that $\pi_{E_i}(a|s) > 0$.
Let us define the following events:
\begin{align*}
\mathcal{E}_p = \bigcap_{(s,a) \in \mathcal{S} \times \mathcal{A}} \left\{ N_t(s,a) \textup{KL}\left( \hat{p}_t(\cdot|s,a), p(\cdot|s,a) \right) \le \log\left(\frac{3SAn}{\delta} \right) + (S-1) \log\left( e\left(1+\frac{N_t(s,a) }{S-1}\right) \right) \right\}
\end{align*}
\begin{align*}
\mathcal{E}_{\pi} = \bigcap_{i \in \left\{ 2, \dots, n+1 \right\} } \bigcap_{s \in \mathcal{S}} \left\{ N_t(s) \textup{KL}\left( \hat{\pi}_{E_i, t}(\cdot|s), \pi_{E_i}(\cdot|s) \right) \le \log\left(\frac{3SAn}{\delta} \right) + (A-1) \log\left( e\left(1+\frac{N_t(s) }{A-1}\right) \right) \right\}
\end{align*}
\begin{align*}
\mathcal{E}_{\pi_{\textup{min}}} = \bigcap_{i \in \left\{ 2, \dots, n+1 \right\} } \bigcap_{(s,a) \in \mathcal{S} \times \mathcal{A}:\pi_{E_i}(a|s)>0} \left\{ \hat{\pi}_{E_i, t}(a|s) \le \pi_{E_i}(a|s) \left( 1 + \sqrt{\frac{3\log\left( \frac{3SAn}{\delta} \right)}{\pi_{E_i}(a|s) N_t(s)}} \right) \right\}
\end{align*}
Consider $\mathcal{E} = \mathcal{E}_p \cap \mathcal{E}_{\pi} \cap \mathcal{E}_{\pi_{\textup{min}}}$. Then, it holds that $\mathbb{P}(\mathcal{E}) \ge 1-\delta$.
\end{lemma}
\begin{proof}
The proof follows by controlling the probability of the complementary event $\mathcal{E}^c$. More specifically, we have that:
\begin{align*}
\mathbb{P}\left( \mathcal{E}^c  \right) = \mathbb{P}\left( \mathcal{E}_p^c \cup \mathcal{E}_{\pi}^c \cup \mathcal{E}_{\pi_{\textup{min}}}^c \right) \le \mathbb{P}\left( \mathcal{E}_p^c \right) + \mathbb{P}\left( \mathcal{E}^c_{\pi} \right) + \mathbb{P}\left( \mathcal{E}^c_{\pi_{\textup{min}}}\right),
\end{align*}
where in the last passage we have used Boole's inequality. At this point, we will use Lemma \ref{lemma:multiplicative-chern} and Lemma \ref{lemma:johnsonn} to control $\mathbb{P}\left( \mathcal{E}_p^c \right) + \mathbb{P}\left( \mathcal{E}^c_{\pi} \right) + \mathbb{P}\left( \mathcal{E}^c_{\pi_{\textup{min}}}\right)$. More specifically, concerning $\mathcal{E}^c_p$, we have that:
\begin{align*}
\mathbb{P} \left( \mathcal{E}^c_p \right) & \le \sum_{(s,a) \in \mathcal{S} \times \mathcal{A}} \mathbb{P} \left( N_t(s,a) \textup{KL}\left( \hat{p}_t(\cdot|s,a), p(\cdot|s,a) \right) > \log\left(\frac{3SAn}{\delta} \right) + (S-1) \log\left( e\left(1+\frac{N_t(s,a) }{S-1}\right) \right) \right) \\ & \le \sum_{(s,a) \in \mathcal{S} \times \mathcal{A}}\frac{\delta}{3SAn} \le \frac{\delta}{3},
\end{align*}
where, in the first step we have applied Boole's inequality, and in the second one Lemma \ref{lemma:johnsonn}. From identical reasoning, we can upper bound $\mathbb{P}\left( \mathcal{E}^c_{\pi} \right)$, thus obtaining:
\begin{align*}
\mathbb{P} \left( \mathcal{E}^c_{\pi} \right) \le \frac{\delta}{3}.
\end{align*}
Finally, concerning $\mathcal{E}^c_{\pi_{\textup{min}}}$, we have that:
\begin{align*}
\mathbb{P} \left( \mathcal{E}^c_{\pi_{\textup{min}}} \right) & \le \sum_{i \in \left\{2, \dots, n+2 \right\}} \sum_{(s,a): \pi_{E_i}(a|s) > 0} \mathbb{P} \left(  \hat{\pi}_{E_i, t}(a|s) > \pi_{E_i}(a|s) \left( 1 + \sqrt{\frac{3\log\left( \frac{3SAn}{\delta} \right)}{\pi_{E_i}(a|s) N_t(s)}} \right) \right) \\ & \le \sum_{i \in \left\{2, \dots, n+2 \right\}} \sum_{(s,a): \pi_{E_i}(a|s) > 0} \frac{\delta}{3SAn} \le \frac{\delta}{3},
\end{align*}
where, in the first step we have used Boole's inequality, and in the second one Lemma \ref{lemma:multiplicative-chern} together with Equation \eqref{eq:good-event-ass}. At this point, combining these results, we obtain $\mathbb{P}\left( \mathcal{E}^c \right) \le \delta$, thus concluding the proof.
\end{proof}

At this point, it has to be remarked that Lemma \ref{lemma:good-event} holds whenever:
\begin{align*}
N_t(s) > \frac{3 \log\left( \frac{3SAn}{\delta} \right)}{\pi_{\textup{min}}}.
\end{align*}
We remark that this term is due to the requirement that Lemma \ref{lemma:multiplicative-chern} requires $\alpha$ in $(0, 1)$. As we shall see, however, this will be sufficient to carry out the theoretical analysis of US-IRL-SE.

We now proceed by presenting error propagation results that study the Hausdorff distance between $\mathcal{R}_{\bar{\mathfrak{B}}}$ and $\mathcal{R}_{\widehat{\bar{\mathfrak{B}}}}$. Before diving into the details, we introduce the following notation. We denote with $\Psi$ the set of $\zeta$'s that are compatible with the linear constraints of Equation \eqref{eq:explicit-eq2}. More specifically:
\begin{align}\label{eq:psi-def}
\Psi \coloneqq  \left\{ \zeta \in \mathbb{R}^{S \times A}_{\ge 0} : \forall i \in \left\{2, \dots, n+1 \right\} \quad d^{\pi_{E_i}} \pi_{E_i} \bar{B}^{\pi_{E_1}} \zeta \le \bm{1}_S \xi_i \right\}.
\end{align} 
Similarly, we denote with $\widehat{\Psi}$, the set of $\zeta$'s that are compatible with the linear constraints of Equation \eqref{eq:explicit-eq2} that are induced by the empirical IRL problem $\widehat{\bar{\mathfrak{B}}}$. More precisely:
\begin{align}\label{eq:psi-hat-def}
\hat{\Psi} \coloneqq  \left\{ \zeta \in \mathbb{R}^{S \times A}_{\ge 0} : \forall i \in \left\{2, \dots, n+1 \right\} \quad \hat{d}^{\hat{\pi}_{E_i}} \hat{\pi}_{E_i} \bar{B}^{\hat{\pi}_{E_1}} \zeta \le \bm{1}_S \xi_i \right\},
\end{align}
where $\hat{d}^{\pi}$ denotes the discounted expected occupancy of policy $\pi$ under the transition model $\hat{p}$.

At this point, we provide a preliminary Lemma that will be used to study the Hausdorff distance.
\begin{lemma}[Error Propagation]\label{lemma:error-prop}
    Let $\bar{\mathfrak{B}}$ be an IRL-SE problem and let $\widehat{\bar{\mathfrak{B}}}$ be its empirical estimate. Then, for any $r \in \mathcal{R}_{\bar{\mathfrak{B}}}$ such that $r = -\bar{B}^{\pi_{E_1}} \zeta + (E-\gamma P)V$, and $(I-\gamma \pi_{E_i} P)^{-1} \pi_{E_i} \bar{B}^{\pi_{E_1}} \zeta \le \xi_i$, there exists $\hat{r} \in \mathcal{R}_{\widehat{\bar{\mathfrak{B}}}}$ such that element-wise it holds that:
    \begin{align}\label{eq:theo-error-prop-eq1}
        |r-\hat{r}| \le \bar{B}^{\pi_{E_1}} B^{\hat{\pi}_{E_1}} \zeta + \gamma |(P-\hat{P})V| + \bar{B}^{\hat{\pi}_{E_1}}\bar{B}^{{\pi}_{E_1}}||\zeta - \textup{proj}_{\widehat{\Psi}}\left( \zeta \right)||_{\infty} \mathbf{1}_{\mathcal{S} \times \mathcal{A}},
    \end{align}
    where $\textup{proj}_{\widehat{\Psi}}(\cdot)$ denotes the infinite norm projection on the set $\widehat{\Psi}$. More formally, for $\zeta \in \mathbb{R}^{\mathcal{S} \times \mathcal{A}}$:
    \begin{align*}
    \textup{proj}_{\widehat{\Psi}}(\zeta) = \argmin_{x \in \widehat{\Psi}} ||x - \zeta||_\infty
    \end{align*}
    
\end{lemma}
\begin{proof}
From Theorem \ref{theorem:multiple-expert-explicit}, we know that we can express the reward functions as:
\begin{align*}
    r = -\bar{B}^{\pi_{E_1}} \zeta + (E-\gamma P)V, \\ \hat{r} = -\bar{B}^{\hat{\pi}_{E_1}} \hat{\zeta} + (E-\gamma \hat{P})\hat{V}, 
\end{align*}
where $V$ and $\hat{V}$ belongs to $\mathbb{R}^{\mathcal{S}}$, and $\zeta$, $\hat{\zeta} \in \mathbb{R}^{\mathcal{S} \times \mathcal{A}}_{\ge 0}$ satisfies the following equations for all $i \in \left\{2, \dots, n+1 \right\}$:
\begin{align*}
    (I - \gamma \pi_{E_i} P)^{-1} \pi_{E_i} \bar{B}^{\pi_{E_1}} \zeta \le \xi_i, \\ (I - \gamma \hat{\pi}_{E_i} \hat{P})^{-1} \hat{\pi}_{E_i} \bar{B}^{\hat{\pi}_{E_1}} \hat{\zeta} \le \xi_i.
\end{align*}
At this point, we recall that, for any $r \in \mathcal{R}_\mathfrak{B}$, we are interested in a specific reward function $\hat{r} \in \mathcal{R}_{\widehat{\mathfrak{B}}}$ that is sufficiently close to $r$.
For this reason, we pick $\hat{V} = V$, and $\hat{\zeta} = \bar{B}^{\pi_{E_1}} \textup{proj}_{ \widehat{\Psi}} (\zeta)$.\footnote{Notice that, if $x \in \widehat{\Psi}$, $\bar{B}^{\pi_{E_1}}x$ belongs to $\widehat{\Psi}$ as well.} Plugging these choices within $|r - \hat{r}|$, and applying the triangular inequality, we obtain:
\begin{align}\label{eq:theo-error-prop-eq3}
    \Big|r - \hat{r}\Big| \le \Big|-(\bar{B}^{\pi_{E_1}} \zeta + \bar{B}^{\hat{\pi}_{E_1}} \hat{\zeta}) \Big| + \gamma \Big|(P-\hat{P})V\Big|.
\end{align}
Let us focus now on $|-(\bar{B}^{\pi_{E_1}} \zeta + \bar{B}^{\hat{\pi}_{E_1}} \hat{\zeta})|$. 
\begin{align*}
    \Big|-(\bar{B}^{\pi_{E_1}} \zeta + \bar{B}^{\hat{\pi}_{E_1}} \bar{B}^{\pi_{E_1}} \textup{proj}_{\widehat{\Psi}} (\zeta))\Big| & = \Big|-(\bar{B}^{\pi_{E_1}} \zeta + \bar{B}^{\hat{\pi}_{E_1}} \bar{B}^{\pi_{E_1}} \textup{proj}_{\widehat{\Psi}} (\zeta)) \pm \bar{B}^{\hat{\pi}_{E_1}} \bar{B}^{\pi_{E_1}} \zeta \Big| \\ & \le \Big| -\bar{B}^{\pi_{E_1}} \zeta + \bar{B}^{\hat{\pi}_{E_1}} \bar{B}^{\pi_{E_1}} \zeta \Big| + \Big| \bar{B}^{\hat{\pi}_{E_1}}\hat{\zeta} -\bar{B}^{\hat{\pi}_{E_1}} \bar{B}^{\pi_{E_1}} \zeta\Big| \\ & = \Big|  -\bar{B}^{{\pi}_{E_1}} B^{\hat{\pi}_{E_1}} \zeta \Big| +  \Big| \bar{B}^{\hat{\pi}_{E_1}} \bar{B}^{\pi_{E_1}}\textup{proj}_{\widehat{\Psi}} (\zeta)) -\bar{B}^{\hat{\pi}_{E_1}} \bar{B}^{\pi_{E_1}} \zeta\Big|,
\end{align*}
where in the last step we have used $B^\pi + \bar{B}^\pi = I_{\mathcal{S} \times \mathcal{A}}$.\footnote{Consider $\mathbb{R}^{\mathcal{S} \times \mathcal{A}}$, then $B^\pi$ is the operator defined as $(B^\pi g)(s,a) = \mathbbm{1} \left\{ \pi(a|s) > 0 \right\}g(s,a)$.} We now focus on the last term: 
\begin{align*}
\Big| \bar{B}^{\hat{\pi}_{E_1}} \bar{B}^{\pi_{E_1}} \textup{proj}_{\widehat{\Omega}} (\zeta)) -\bar{B}^{\hat{\pi}_{E_1}} \bar{B}^{\pi_{E_1}} \zeta\Big| & = \bar{B}^{\hat{\pi}_{E_1}} \bar{B}^{\pi_{E_1}} \big| \zeta - \textup{proj}_{\widehat{\Omega}}(\zeta) \big| \\ & \le \bar{B}^{\hat{\pi}_{E_1}} \bar{B}^{\pi_{E_1}} || \zeta - \textup{proj}_{\widehat{\Omega}}(\zeta) ||_\infty \mathbf{1}_{\mathcal{S} \times \mathcal{A}}.
\end{align*}
Plugging these results within Equation \eqref{eq:theo-error-prop-eq3} concludes the proof.
\end{proof}
Before providing some interpretation to Lemma \ref{lemma:error-prop}, we recall the definition of the Hausdorff distannce we are interested in:
\begin{align}\label{eq:hauss-reward}
H_\infty(\mathcal{R}_{\bar{\mathfrak{B}}},\mathcal{R}_{\widehat{\bar{\mathfrak{B}}}}) =\max\left\{ \sup_{r \in \mathcal{R}_{\bar{\mathfrak{B}}}} \inf_{\hat{r} \in \mathcal{R}_{\widehat{\bar{\mathfrak{B}}}}} ||r-\hat{r}||_\infty,  \sup_{\hat{r} \in \mathcal{R}_{\widehat{\bar{\mathfrak{B}}}}} \inf_{r \in \mathcal{R}_{\bar{\mathfrak{B}}}} ||r-\hat{r}||_\infty  \right\}.
\end{align}
Given this definition, we can appreciate how Lemma \ref{lemma:error-prop} can be used to upper-bound the error of the first component of the Hausdorff distance, namely:
\begin{align*}
\sup_{r \in \mathcal{R}_{\bar{\mathfrak{B}}}} \inf_{\hat{r} \in \mathcal{R}_{\widehat{\bar{\mathfrak{B}}}}} ||r-\hat{r}||_\infty.
\end{align*}
Nevertheless, it is possible to obtain a symmetric version of Lemma \ref{lemma:error-prop} that analyzes the existence of an $r \in \mathcal{R}_{\bar{\mathfrak{B}}}$ that is sufficiently close to any $\hat{r} \in \mathcal{R}_{\widehat{\bar{\mathfrak{B}}}}$. At this point, we notice that Lemma \ref{lemma:error-prop} decomposes the error on the recovered reward set as the sum of three different components. The first one is related to the error in the estimation of the optimal policy $\pi_{E_1}$, and, to zero it out, it is sufficient to estimate for each state one action that is played by the optimal expert. The second one, instead, is related to the errors in the estimation of the transition model, and, finally, the last one is related to the presence of multiple and sub-optimal experts. More precisely, we notice that the first two components are present also for the single-agent setting \citep{metelli2021provably}, and, in this sense, they arise from the difficulty of estimating reward functions that are compatible with Equation \eqref{eq:explicit-eq1}. The last term, instead, is related to the presence of sub-optimal experts, and it denotes the infinite norm between any $\zeta \in \Psi$ and its infinite-norm projection to the set $\widehat{\Psi}$. As one can imagine, and as our proofs will reveal, studying how this last source of error decreases with the iterations of US-IRL-SE introduces significant challenges in the analysis w.r.t. the single agent setting. Precisely, this complexity will be tackled within the proof of the following Lemma, which analyzes how the error of Equation \eqref{eq:theo-error-prop-eq1} decreases with the number of iterations of US-IRL.

\begin{lemma}[High-Probability Error Propagation]\label{lemma:error-prop-high-prob}
Let $t$ be the iteration of US-IRL. Suppose that $N_t(s,a) \ge 1$ for all $(s,a) \in \mathcal{S} \times \mathcal{A}$ and, furthermore, suppose that $t$ is such that:
\begin{align}\label{eq:time-pi}
t \ge \frac{3 \log\left(3SAn \delta^{-1} \right)}{A \pi_{\textup{min}}}
\end{align}
\begin{align}\label{eq:time-S}
t \ge \frac{8\gamma^2}{(1-\gamma)^2} \left[ \log\left( \frac{3SAn}{\delta} \right)  + (S-1) \log\left( \frac{64 \gamma^4}{(1-\gamma)^4} \left(  \log \left( \frac{3SAn}{\delta} \right) + (S-1) \left( \sqrt{e} + \sqrt{\frac{1}{S-1}} \right)^2 \right) \right) \right]
\end{align}
\begin{align}\label{eq:time-A}
t \ge \frac{8\gamma^2}{A(1-\gamma)^2} \left[ \log\left( \frac{3SAn}{\delta} \right)  + (A-1) \log\left( \frac{64 \gamma^4}{(1-\gamma)^4} \left(  \log \left( \frac{3SAn}{\delta} \right) + (A-1) \left( \sqrt{e} + \sqrt{\frac{1}{A-1}} \right)^2 \right) \right) \right].
\end{align}
Then let $\bar{\mathfrak{B}}$ be an IRL-SE problem and let $\widehat{\bar{\mathfrak{B}}}$ be its empirical estimate after $t$ iteration of US-IRL. Then, with probability at least $1-\delta$, for any $r \in \mathcal{R}_{\bar{\mathfrak{B}}}$, there exists $\hat{r} \in \mathcal{R}_{\widehat{\bar{\mathfrak{B}}}}$ such that:
\begin{align}\label{eq:error-prop-hp-eq1}
||r-\hat{r}||_\infty \le \frac{2\sqrt{2}\gamma}{1-\gamma} \beta_t + \left( \rho_t + {\frac{2\sqrt{2}\gamma}{1-\gamma} \left(\alpha_t + \beta_t \right)} \right) \min\left\{ \pi_{\textup{min}}^{-1} \max_i \xi_i, (1-\gamma)^{-1} \right\},
\end{align}
where:
\begin{align}
& \beta_t \coloneqq \sqrt{\frac{\log\left( \frac{3SAn}{\delta} \right) + (S-1) \log\left(e \left(1+\frac{t}{S-1} \right) \right)}{t}} \\
& \alpha_t \coloneqq \sqrt{\frac{\log\left( \frac{3SAn}{\delta} \right) + (A-1) \log\left(e \left(1+\frac{t A}{A-1} \right) \right)}{t A}} \\ 
& \rho_t \coloneqq \sqrt{\frac{3 \log\left( \frac{3SAn}{\delta} \right)}{\pi_{\textup{min}} t A}}.
\end{align}
\end{lemma}
\begin{proof}
The proof follows by analyzing in greater detail the result of Lemma \ref{lemma:error-prop}. Indeed, from Lemma \ref{lemma:error-prop}, we know that there exists $\hat{r} \in \mathcal{R}_{\widehat{\bar{\mathfrak{B}}}}$ such that:
\begin{align}\label{eq:error-prop-hp-eq2}
||r-\hat{r}||_\infty \le ||\bar{B}^{\pi_{E_1}} B^{\hat{\pi}_{E_1}} \zeta||_\infty + \gamma ||(P-\hat{P})V||_\infty + ||\zeta - \textup{proj}_{\widehat{\Psi}}\left( \zeta \right)||_{\infty}
\end{align}
We split the analysis of Equation \eqref{eq:error-prop-hp-eq2} into two three parts. 
First of all, we will focus on the two first terms, that are $||\bar{B}^{\pi_{E_1}} B^{\hat{\pi}_{E_1}} \zeta||_\infty$ and $\gamma ||(P-\hat{P})V||_\infty$, and then we will tackle the most challenging aspect, that is $||\zeta - \textup{proj}_{\widehat{\Psi}}\left( \zeta \right)||_{\infty}$.

Concerning $||\bar{B}^{\pi_{E_1}} B^{\hat{\pi}_{E_1}} \zeta||_\infty$, we notice that, whenever $N_t(s,a) > 1$ holds, we have that:
\begin{align*}
\bar{B}^{\pi_{E_1}} B^{\hat{\pi}_{E_1}} \zeta = 0,
\end{align*}
for any possible value of $\zeta$. This is a direct consequence of the fact that $\pi_{E_1}$ is deterministic.

Secondly, let us focus on $\gamma ||(P-\hat{P})V||_\infty$. First of all, we notice that, since $N_t(s) > 3 \log\left(3SAn \delta^{-1} \right) \pi_{\textup{min}}$ holds for all $s \in \mathcal{S}$, then $\mathcal{E}$ holds with probability at least $1-\delta$ (Lemma \ref{lemma:good-event}). At this point, conditioning on $\mathcal{E}$, we have that: 
\begin{align*}
\gamma ||(P-\hat{P})V||_\infty & = \gamma \max_{s,a} \Big| \sum_{s' \in \mathcal{S}} (p(s'|s,a) - \hat{p}(s'|s,a ))V(s')  \Big| \\ & \le \frac{\gamma}{1-\gamma} \max_{s,a} \Big| \sum_{s' \in \mathcal{S}} (p(s'|s,a) - \hat{p}(s'|s,a ))  \Big| \\ & \le \frac{\gamma}{1-\gamma} \max_{s,a} || p(\cdot|s,a) - \hat{p}(\cdot|s,a) ||_1 \\ & \le \frac{2\sqrt{2}\gamma}{(1-\gamma)} \max_{s,a} \sqrt{\textup{KL}(\hat{p}(\cdot|s,a), p(\cdot|s,a))} \\ & \le \frac{2\sqrt{2}\gamma}{1-\gamma} \max_{s,a} \sqrt{\frac{\log\left( \frac{3SAn}{\delta} \right) + (S-1) \log\left(e \left(1+\frac{N_t(s,a)}{S-1} \right) \right)}{N_t(s,a)}},
\end{align*}
where the first inequality follows the fact that $||V||_\infty \le \frac{1}{1-\gamma}$, the third one by Pinksker's inequality, and the last one from Lemma \ref{lemma:good-event}. At this point, since US-IRL gathers samples uniformly from each state-action pair, we have that $N_t(s,a) = t$ for all $(s,a) \in \mathcal{S} \times \mathcal{A}$. Thus leading to:
\begin{align*}
\gamma ||(P-\hat{P})V||_\infty \le \frac{2\sqrt{2}\gamma}{1-\gamma} \sqrt{\frac{\log\left( \frac{3SAn}{\delta} \right) + (S-1) \log\left(e \left(1+\frac{t}{S-1} \right) \right)}{t}} \coloneqq \frac{2\sqrt{2}\gamma}{1-\gamma} \beta_t.
\end{align*}
Finally, we focus on $||\zeta - \textup{proj}_{\widehat{\Psi}}\left( \zeta \right)||_{\infty}$. To upper-bound this last term we proceed in several steps. 

\paragraph{Step 1: Relationship between $\Psi$ and $\widehat{\Psi}$}
First of all, we begin by finding a more explicit relationship between $\Psi$ and $\widehat{\Psi}$. To this end, we recall that $\zeta \in \mathbb{R}^{S \times A}_{\ge 0}$ belongs to $\widehat{\Psi}$ if the following condition is satisfied for all $i \in \left\{2, \dots, n+1 \right\}$:
\begin{align*}
\hat{d}^{\hat{\pi}_{E_i}} \hat{\pi}_{E_i} \bar{B}^{\hat{\pi}_{E_1}} \zeta \le \bm{1}_{\mathcal{S}} \xi_i.
\end{align*}
Under the assumption that $N_t(s,a) \ge 1$, since the expert policy $\pi_{E_1}$ is deterministic, the previous Equation can be equivalently rewritten as:
\begin{align*}
\hat{d}^{\hat{\pi}_{E_i}} \hat{\pi}_{E_i} \bar{B}^{{\pi}_{E_1}} \zeta \le \bm{1}_{\mathcal{S}} \xi_i.
\end{align*}
At this point, we proceed with some algebraic manipulations of the left-hand side of the previous Equation. Specifically:
\begin{align*}
\hat{d}^{\hat{\pi}_{E_i}} \hat{\pi}_{E_i} \bar{B}^{{\pi}_{E_1}} \zeta & = \left( \hat{d}^{\hat{\pi}_{E_i}} \pm d^{\pi_{E_i}}\right) \hat{\pi}_{E_i} \bar{B}^{{\pi}_{E_1}} \zeta \\ & = \left( \hat{d}^{\hat{\pi}_{E_i}} - d^{\pi_{E_i}}\right) \hat{\pi}_{E_i} \bar{B}^{{\pi}_{E_1}} \zeta + d^{\pi_{E_i}} \hat{\pi}_{E_i} \bar{B}^{\pi_{E_1}} \zeta \\ & = \left( \hat{d}^{\hat{\pi}_{E_i}} - d^{\pi_{E_i}}\right) \hat{\pi}_{E_i} \bar{B}^{{\pi}_{E_1}} \zeta + d^{\pi_{E_i}} \left( \hat{\pi}_{E_i} \pm \pi_{E_i}  \right) \bar{B}^{\pi_{E_1}} \zeta \\ & = \left( \hat{d}^{\hat{\pi}_{E_i}} - d^{\pi_{E_i}}\right) \hat{\pi}_{E_i} \bar{B}^{{\pi}_{E_1}} \zeta + d^{\pi_{E_i}} \left( \hat{\pi}_{E_i} - \pi_{E_i}  \right) \bar{B}^{\pi_{E_1}} \zeta + d^{\pi_{E_i}} \pi_{E_i} \bar{B}^{\pi_{E_1}} \zeta 
\end{align*}
At this point, focus on $\hat{d}^{\hat{\pi}_{E_i}} - d^{\pi_{E_i}}$:
\begin{align*}
\hat{d}^{\hat{\pi}_{E_i}} - d^{\pi_{E_i}} & = (I_\mathcal{S} - \gamma \hat{\pi}_{E_i}\hat{P})^{-1}I_{\mathcal{S}} - (I_\mathcal{S} - \gamma {\pi}_{E_i}{P})^{-1} I_{\mathcal{S}} \\ & = (I_\mathcal{S} - \gamma {\pi}_{E_i} {P})^{-1} \left[ (I_\mathcal{S} - \gamma {\pi}_{E_i} P ) - (I_\mathcal{S} - \gamma \hat{\pi}_{E_i} \hat{P} ) \right] \hat{d}^{\hat{\pi}_{E_i}} \\ & = \gamma d^{\pi_{E_i}} \left[ \hat{\pi}_{E_i} \hat{P} - \pi_{E_i} P \right] \hat{d}^{\hat{\pi}_{E_i}}
\end{align*}

At this point, plugging this result into the previous Equation, we obtain that $\zeta$ belongs to $\widehat{\Psi}$ if and only if the following holds for all experts $i \in \left\{2, \dots, n+1 \right\}$:
\begin{align}\label{eq:error-prop-hp-eq4}
d^{\pi_{E_i}} \pi_{E_i} \bar{B}^{\pi_{E_1}} \zeta  + \gamma d^{\pi_{E_i}} \left[ \hat{\pi}_{E_i} \hat{P} - \pi_{E_i} P \right] \hat{d}^{\hat{\pi}_{E_i}} \hat{\pi}_{E_i} \bar{B}^{{\pi}_{E_1}} \zeta + d^{\pi_{E_i}} \left( \hat{\pi}_{E_i} - \pi_{E_i}  \right) \bar{B}^{\pi_{E_1}} \zeta \le \bm{1}_{\mathcal{S}}\xi_i.
\end{align}
As we can appreciate, Equation \eqref{eq:error-prop-hp-eq4} relates the sets $\Psi$ and $\widehat{\Psi}$ by an explicit relationship. Indeed, by renaming $\epsilon_1(\zeta) \coloneqq \gamma d^{\pi_{E_i}} \left[ \hat{\pi}_{E_i} \hat{P} - \pi_{E_i} P \right] \hat{d}^{\pi} \hat{\pi}_{E_i} \bar{B}^{{\pi}_{E_1}} \zeta$ and $\epsilon_2(\zeta) \coloneqq d^{\pi_{E_i}} \left( \hat{\pi}_{E_i} - \pi_{E_i}  \right) \bar{B}^{\pi_{E_1}} \zeta$, Equation \eqref{eq:error-prop-hp-eq4} can be rewritten as:
\begin{align*}
d^{\pi_{E_i}} \pi_{E_i} \bar{B}^{\pi_{E_1}} \zeta + \epsilon_1(\zeta) + \epsilon_{2}(\zeta) \le \bm{1}_{\mathcal{S}}\xi_i,
\end{align*}
which closely resambles the definition of $\Psi$.

\paragraph{Step 2: Restricting the set $\widehat{\Psi}$} We continue our proof by defining a new set $\widetilde{\Psi}$ such that $\widetilde{\Psi} \subseteq \widehat{\Psi}$ with probability at least $1-\delta$. Indeed, by definition of the projection according to the infinite norm, this allows to upper-bound $||\zeta - \textup{proj}_{\widehat{\Psi}}\left( \zeta \right)||_{\infty}$. More specifically, for any $\widetilde{\Psi} \subseteq \widehat{\Psi}$, we have that:
\begin{align}\label{eq:error-prop-hp-eq5}
||\zeta - \textup{proj}_{\widehat{\Psi}}\left( \zeta \right)||_{\infty} \le ||\zeta - \textup{proj}_{\widetilde{\Psi}}\left( \zeta \right)||_{\infty}.
\end{align}
More specifically, in order to define $\widetilde{\Psi}$, we will first proceed by upper bounding $\epsilon_1(\zeta)$ and $\epsilon_{2}(\zeta)$. Precisely, consider $\zeta \in \widehat{\Psi}$. For $\epsilon_1(\zeta)$ we have that:
\begin{align*}
\epsilon_1(\zeta) & = \gamma d^{\pi_{E_i}} \left[ \hat{\pi}_{E_i} \hat{P} - \pi_{E_i} P \right] \hat{d}^{\hat{\pi}_{E_i}} \hat{\pi}_{E_i} \bar{B}^{{\pi}_{E_1}} \zeta \\ & = \gamma \left[ d^{\pi_{E_i}} \left( \hat{\pi}_{E_i} - \pi_{E_i} \right) \hat{P} \right] \hat{d}^{\hat{\pi}_{E_i}} \hat{\pi}_{E_i} \bar{B}^{\pi_{E_1}} \zeta + \gamma \left[ d^{\pi_{E_i}} \pi_{E_i} (\hat{P} - P) \right] \hat{d}^{\hat{\pi}_{E_i}} \hat{\pi}_{E_i} \bar{B}^{\pi_{E_1}} \\ & \le \gamma \left[ d^{\pi_{E_i}} \Big| \hat{\pi}_{E_i} - \pi_{E_i} \Big| \hat{P} \right] \hat{d}^{\hat{\pi}_{E_i}} \hat{\pi}_{E_i} \bar{B}^{\pi_{E_1}} \zeta + \gamma \left[ d^{\pi_{E_i}} \pi_{E_i} \Big|\hat{P} - P\Big| \right] \hat{d}^{\hat{\pi}_{E_i}} \hat{\pi}_{E_i} \bar{B}^{\pi_{E_1}},
\end{align*}
where, given $f \in \mathbb{R}^S$, and $g \in \mathbb{R}^{S \times A}$, $\Big| \hat{\pi} - \pi \Big|$ denotes the operator defined as $\Big| \hat{\pi} - \pi \Big|g(s,a) = \sum_{a} \Big|\hat{\pi}(a|s) - \pi(a|s)\Big| g(s,a)$, and, similarly $\Big| \hat{P} - P \Big|f(s) = \sum_{s'} \Big|\hat{p}(s'|s,a) - p(s'|s,a)\Big| f(s')$. At this point, we notice that, since, for all $\zeta \in \widehat{\Psi}$, it holds that $\hat{d}^{\hat{\pi}_{E_i}} \hat{\pi}_{E_i} \bar{B}^{\pi_{E_1}} \le \bm{1}_{\mathcal{S}} \xi_i$. Therefore, we can further upper bound the previous equation to obtain:
\begin{align}\label{eq:error-prop-hp-eq6}
\epsilon_1(\zeta) & \le \gamma \Big|\Big|\left[ d^{\pi_{E_i}} \Big| \hat{\pi}_{E_i} - \pi_{E_i} \Big| \right] \bm{1}_{\mathcal{S} \times \mathcal{A}} \xi_i \Big|\Big|_\infty \bm{1}_{\mathcal{S}} + \gamma \Big|\Big| \left[ d^{\pi_{E_i}} \pi_{E_i} \Big|\hat{P} - P\Big| \right] \bm{1}_{\mathcal{S}} \xi_i\Big|\Big|_\infty \bm{1}_\mathcal{S}.
\end{align}
At this point, let us focus on $\gamma \Big|\Big| \left[ d^{\pi_{E_i}} \Big| \hat{\pi}_{E_i} - \pi_{E_i} \Big| \right] \bm{1}_{\mathcal{S} \times \mathcal{A}} \xi_i\Big|\Big|_\infty$:
\begin{align*}
\gamma \Big|\Big|\left[ d^{\pi_{E_i}} \Big| \hat{\pi}_{E_i} - \pi_{E_i} \Big| \right] \bm{1}_{\mathcal{S} \times \mathcal{A}} \xi_i \Big|\Big|_\infty & \le \gamma \xi_i \max_{s'} \sum_{s} d^{\pi_{E_i}}_{s'}(s) \sum_{a} \Big| \hat{\pi}_{E_i}(a|s) - \pi_{E_i}(a|s) \Big|  \\ & = \gamma \xi_i \max_{s'} \sum_{s} d^{\pi_{E_i}}_{s'}(s) ||\hat{\pi}_{E_i}(\cdot|s) - \pi_{E_i}(\cdot|s) ||_1 \\ & \le \frac{2 \sqrt{2} \gamma \xi_i}{1-\gamma} \max_{s'} \sqrt{\textup{KL}(\hat{\pi}_{E_i}(\cdot, s'), \pi_{E_i}(\cdot, s')) } \\ & \le \frac{2 \sqrt{2} \gamma \xi_i}{1-\gamma} \max_{s} \sqrt{\frac{\log\left( \frac{3SAn}{\delta} \right) + (A-1) \log\left(e \left(1+\frac{N_t(s)}{S-1} \right) \right)}{N_t(s)}} \\ & = \frac{2 \sqrt{2} \gamma \xi_i}{1-\gamma} \sqrt{\frac{\log\left( \frac{3SAn}{\delta} \right) + (A-1) \log\left(e \left(1+\frac{t A}{A-1} \right) \right)}{t A}} \\ & \coloneqq \frac{2 \sqrt{2} \gamma \xi_i}{1-\gamma} \alpha_t,  
\end{align*} 
where in the third step we have used Pinkser's inequality, in the fourth we have used Lemma \ref{lemma:good-event}, and in the fifth one we have used the fact that, in US-IRL-SE, $N_t(s) = \sum_{a \in \mathcal{A}} N_t(s,a) = t A$. 
At this point, focus on the second term of Equation \eqref{eq:error-prop-hp-eq6}, namely $\gamma \Big|\Big| \left[ d^{\pi_{E_i}} \pi_{E_i} \Big|\hat{P} - P\Big| \right] \bm{1}_{\mathcal{S}} \xi_i\Big|\Big|_\infty \bm{1}_\mathcal{S}$:
\begin{align*}
\gamma \Big|\Big| \left[ d^{\pi_{E_i}} \pi_{E_i} \Big|\hat{P} - P\Big| \right] \bm{1}_{\mathcal{S}} \xi_i\Big|\Big|_\infty & \le \gamma \xi_i \max_{s'} \sum_{s} d^{\pi_{E_i}}_{s'}(s) \sum_{a} \pi_{E_i}(a|s) ||p(\cdot|s,a) - \hat{p}(\cdot|s,a)||_1 \\ & \le \frac{2\sqrt{2}\gamma \xi_i}{1-\gamma} \max_{s,a} \sqrt{\textup{KL}(\hat{p}(\cdot|s,a) - p(\cdot|s,a))} \\ & \le \frac{2\sqrt{2}\gamma \xi_i}{1-\gamma} \max_{s,a} \sqrt{\frac{\log\left( \frac{3SAn}{\delta} \right) + (S-1) \log\left(e \left(1+\frac{N_t(s)}{S-1} \right) \right)}{N_t(s)}} \\ & \le \frac{2\sqrt{2}\gamma \xi_i}{1-\gamma} \sqrt{\frac{\log\left( \frac{3SAn}{\delta} \right) + (S-1) \log\left(e \left(1+\frac{t}{S-1} \right) \right)}{t}} \\ & \coloneqq \frac{2\sqrt{2}\gamma \xi_i}{1-\gamma}  \beta_t,
\end{align*}
where in the second step we have used Pinkser's inequality, in the third one Lemma \ref{lemma:good-event}, and in the fourth one we have used the fact that, in US-IRL-SE, $N_t(s)=t$ for all states $s \in \mathcal{S}$. Therefore, plugging these results within Equation \eqref{eq:error-prop-hp-eq6}, we obtain an high-probability upper bound on $\epsilon_1(\zeta)$, that is:
\begin{align}\label{eq:error-prop-hp-eq7}
||\epsilon_1(\zeta)||_\infty \le \frac{2\sqrt{2}\gamma \xi_i}{1-\gamma}  \left( \alpha_t + \beta_t \right).
\end{align}
We now proceed with similar reasoning to obtain an upper bound on $\epsilon_2(\zeta)$. Nevertheless, contrary to $\epsilon_1(\zeta)$, here we proceed with an element-wise upper bound on $\epsilon_2(\zeta)$. Specifically, for each state $s' \in \mathcal{S}$:
\begin{align*}
\epsilon_2(\zeta)(s') & = \sum_{s} d^{\pi_{E_i}}_{s'}(s) \sum_{a: \pi_{E_1}(a|s) = 0} \Big| \hat{\pi}_{E_i}(a|s) - \pi_{E_i}(a|s)\Big| \zeta(s,a) \\ & = \sum_{s} d^{\pi_{E_i}}_{s'}(s) \sum_{\substack{a: \pi_{E_1}(a|s) = 0, \\ \pi_{E_i}(a|s) > 0}} \Big| \hat{\pi}_{E_i}(a|s) - \pi_{E_i}(a|s)\Big| \zeta(s,a) \\ & = \sum_{s} d^{\pi_{E_i}}_{s'}(s) \sum_{\substack{a: \pi_{E_1}(a|s) = 0, \\ \pi_{E_i}(a|s) > 0}} \Big| \frac{\hat{\pi}_{E_i}(a|s) - \pi_{E_i}(a|s)}{\pi_{E_i}(a|s)}\Big| \pi_{E_i}(a|s) \zeta(s,a) \\ & \le \max_{\substack{(s'',a''): \pi_{E_i}(a''|s'') > 0}} \Big| \frac{\hat{\pi}_{E_i}(a''|s'') - \pi_{E_i}(a''|s'')}{\pi_{E_i}(a''|s'')} \Big| \sum_{s} d^{\pi_{E_i}}_{s'}(s) \sum_{a: \pi_{E_1}(a|s) = 0} \pi_{E_i}(a|s) \zeta(s,a) \\ & \le \sqrt{\frac{3 \log\left( \frac{3SAn}{\delta} \right)}{\pi_{\textup{min}} N_t(s)}} \sum_{s} d^{\pi_{E_i}}_{s'}(s) \sum_{a: \pi_{E_1}(a|s) = 0} \pi_{E_i}(a|s) \zeta(s,a) \\  & = \sqrt{\frac{3 \log\left( \frac{3SAn}{\delta} \right)}{\pi_{\textup{min}} t A}} \sum_{s} d^{\pi_{E_i}}_{s'}(s) \sum_{a: \pi_{E_1}(a|s) = 0} \pi_{E_i}(a|s) \zeta(s,a),
\end{align*}
where the third step follows from the fact that, since $N_t(s) \ge 1$, $\hat{\pi}_{E_i}(a|s) = 0$ for all $(s,a)$ such that $\pi_{E_i}(a|s) = 0$, the fifth one, instead, from Lemma \ref{lemma:good-event}, and the last one from the fact that, in US-IRL-SE, $N_t(s) = \sum_a N_t(s,a) = t A $.\footnote{Notice that Equation \eqref{eq:time-pi} guarantees that we can apply Lemma \ref{lemma:good-event}.} Therefore, we have obtained an element-wise upper-bound on $\epsilon_2(\zeta)$ of the following form:
\begin{align}\label{eq:error-prop-hp-eq8}
\epsilon_2(\zeta) \le \sqrt{\frac{3 \log\left( \frac{3SAn}{\delta} \right)}{\pi_{\textup{min}} t A}} d^{\pi_{E_i}} \pi_{E_i} \bar{B}^{\pi_{E_1}} \zeta \coloneqq \rho_t d^{\pi_{E_i}} \pi_{E_i} \bar{B}^{\pi_{E_1}} \zeta.
\end{align}

We are finally ready to define the set $\widetilde{\Psi}$. More specifically:
\begin{align}
\widetilde{\Psi} \coloneqq \left\{ \zeta \in \mathbb{R}^{S \times A}_{\ge 0} : (1+\rho_t) d^{\pi_{E_i}} \pi_{E_i} \bar{B}^{\pi_{E_1}} \zeta  \le \xi_i - \frac{2\sqrt{2}\gamma\xi_i}{1-\gamma} \left( \alpha_t + \beta_t \right) \right\}. 
\end{align}
As one can verify, the definition of $\widetilde{\Psi}$ follows from upper-bounding $\epsilon_1(\zeta)$ and $\epsilon_2(\zeta)$ with Equations \eqref{eq:error-prop-hp-eq7} and \eqref{eq:error-prop-hp-eq8}. As a direct consequence, whenever $\zeta \in \widetilde{\Psi}$, we have that $\zeta$ belongs to $\widehat{\Psi}$ as well.

\paragraph{Step 3: Ensuring that the feasible region of $\widetilde{\Psi}$ is non-empty}
At this point, one might be tempted to directly study the projection of $||\zeta - \textup{proj}_{\widetilde{\Psi}}\left( \zeta \right)||_{\infty}$. Nevertheless, we notice that, for sufficiently large values of $\alpha_t$ and $\beta_t$, $\widetilde{\Psi}$ might be empty. 

Sufficient conditions to guarantee that $\widetilde{\Psi}$ is not empty are the following ones:
\begin{align}
& \frac{2\sqrt{2}\gamma \xi_i}{1-\gamma} \alpha_t\le \frac{\xi_i}{2}\label{eq:error-prop-hp-eq100} \\
& \frac{2\sqrt{2}\gamma \xi_i}{1-\gamma} \beta_t \le \frac{\xi_i}{2}\label{eq:error-prop-hp-eq101}.
\end{align}

Using Lemma 12 in \citep{jonsson2020planning}, it is possible to verify that Equations \eqref{eq:time-S} and \eqref{eq:time-A} are sufficient conditions for Equations \eqref{eq:error-prop-hp-eq100}-\eqref{eq:error-prop-hp-eq101} to hold.

\paragraph{Step 4: Picking $\tilde{\zeta} \in \widetilde{\Psi}$ to upper-bound $||\zeta - \textup{proj}_{\widehat{\Psi}}\left( \zeta \right)||_{\infty}$}

At this point, we are ready to conclude our proof. As we have previously verified, the set $\widetilde{\Psi}$ is non-empty. We can now study $||\zeta - \textup{proj}_{\widehat{\Psi}}\left( \zeta \right)||_{\infty}$. More precisely, we have that:
\begin{align*}
||\zeta - \textup{proj}_{\widehat{\Psi}}\left( \zeta \right)||_{\infty} \le ||\zeta - \textup{proj}_{\widetilde{\Psi}}\left( \zeta \right)||_{\infty}.
\end{align*}
To further upper bound this Equation, we notice that we can always pick, by definition of the infinite norm projection, any $\tilde{\zeta} \in \widetilde{\Psi}$. In other words, we have that:
\begin{align*}
||\zeta - \textup{proj}_{\widehat{\Psi}}\left( \zeta \right)||_{\infty} \le ||\zeta - \textup{proj}_{\widetilde{\Psi}}\left( \zeta \right)||_{\infty} \le ||\zeta - \tilde{\zeta} ||_\infty.
\end{align*}
More specifically, we choose $\tilde{\zeta}$ in the following way. If for all $i \in \left\{ 2, \dots, n+1 \right\}$, $\pi_{E_i}(a|s) = 0$, then we pick $\tilde{\zeta}(s,a) = \zeta(s,a)$;\footnote{Notice that, if for all $i \in \left\{ 2, \dots, n+1 \right\}$, $\pi_{E_i}(a|s) = 0$, $\zeta(s,a)$ does not contribute to any of the linear constraints that introduced by Equation \eqref{eq:explicit-eq2}.} otherwise, we pick $\zeta(s,a)$ in the following way:

\begin{align}\label{eq:error-prop-hp-eq9}
\tilde{\zeta}(s,a) = \frac{1-\frac{2\sqrt{2}\gamma}{1-\gamma} \left(\alpha_t + \beta_t\right)}{1+\rho_t} \zeta(s,a).
\end{align}

First of all, we verify that this choice of $\tilde{\zeta}$ belongs to $\widetilde{\Psi}$. Plugging Equation \eqref{eq:error-prop-hp-eq9} into the definition of $\widetilde{\Psi}$ we obtain that:
\begin{align*}
\left(1-\frac{2\sqrt{2}\gamma}{1-\gamma} \left(\alpha_t + \beta_t\right)\right) \frac{1+\rho_t}{1+\rho_t}  d^{\pi_{E_i}} \pi_{E_i} \bar{B}^{\pi_{E_1}} \zeta  \le \xi_i - \frac{2\sqrt{2}\gamma\xi_i}{1-\gamma} \left( \alpha_t + \beta_t \right).
\end{align*}
However, since $\zeta \in \Psi$, we have that $d^{\pi_{E_i}} \pi_{E_i} \bar{B}^{\pi_{E_1}} \zeta \le \xi_i$, thus leading to:
\begin{align*}
\left(\xi_i -\frac{2\sqrt{2}\gamma \xi_i}{1-\gamma} \left(\alpha_t + \beta_t\right)\right)  \le \xi_i - \frac{2\sqrt{2}\gamma\xi_i}{1-\gamma} \left( \alpha_t + \beta_t \right),
\end{align*}
which is always true, and, consequently $\tilde{\zeta} \in \widetilde{\Psi}$.

To conclude the proof, it remains to analyze $||\zeta - \tilde{\zeta}||_\infty$. At this point, we notice that:
\begin{align*}
|| \zeta - \tilde{\zeta}||_{\infty}& = \max_{\substack{(s,a): \pi_{E_1(a|s)} = 0, \\ \exists i : \pi_{E_i}(a|s) > 0}} \Big| \zeta(s,a) - \frac{1-\frac{2\sqrt{2}\gamma}{1-\gamma} \left(\alpha_t + \beta_t \right)}{1+\rho_t} \zeta(s,a) \Big|,
\end{align*}
Indeed, for a state-action pair $(s,a)$ such that $\pi_{E_1}(a|s) = 0$, we can notice that any value $\zeta(s,a)$  will not affect the resulting reward function.\footnote{In other words, if $(s,a)$ is such that $\pi_{E_1}(a|s) = 0$, we can add the following constraint to any of the set we defined: $\zeta(s,a)=0$. The resulting feasible reward set is left unchanged.} Furthermore, whenever there is no sub-optimal expert such that $\pi_{E_i}(a|s) > 0$, then $\tilde{\zeta}(s,a) = \zeta(s,a)$ holds by definition. To conclude, with simple algebraic manipulations we can obtain the following result:
\begin{align*}
|| \zeta - \tilde{\zeta}||_{\infty} & \le \max_{\substack{(s,a): \pi_{E_1(a|s)} = 0, \\ \exists i : \pi_{E_i}(a|s) > 0}} \Big| (1+\rho_t) \zeta(s,a) - \left( {1-\frac{2\sqrt{2}\gamma}{1-\gamma} \left(\alpha_t + \beta_t \right)} \right)\zeta(s,a) \Big| \\ & = \max_{\substack{(s,a): \pi_{E_1(a|s)} = 0, \\ \exists i : \pi_{E_i}(a|s) > 0}} \Big| \left( \rho_t + {\frac{2\sqrt{2}\gamma}{1-\gamma} \left(\alpha_t + \beta_t \right)} \right) \zeta(s,a) \Big| \\ & = \left( \rho_t + {\frac{2\sqrt{2}\gamma}{1-\gamma} \left(\alpha_t + \beta_t \right)} \right) \max_{\substack{(s,a): \pi_{E_1(a|s)} = 0, \\ \exists i : \pi_{E_i}(a|s) > 0}} \Big|  \zeta(s,a) \Big| \\ & \le \left( \rho_t + {\frac{2\sqrt{2}\gamma}{1-\gamma} \left(\alpha_t + \beta_t \right)} \right) \min\left\{ \pi_{\textup{min}}^{-1} \max_i \xi_i, (1-\gamma)^{-1} \right\},
\end{align*}
where in the last step we have upper-bounded $\zeta(s,a)$ with an upper bound that follows directly from Equation \eqref{eq:upper-bound}, thus concluding the proof.
\end{proof}

We can appreciate as Lemma \ref{lemma:error-prop-high-prob} provides, under certain conditions on the time-step $t$, a high-probability upper bound on the difference $\inf_{\hat{r} \in \mathcal{R}_{\widehat{\bar{\mathfrak{B}}}}}||r - \hat{r}||_\infty$ for any choice of $r \in \mathcal{R}_{\bar{\mathfrak{B}}}$.
 Furthermore, as we can see, the proof is fairly involved due to the necessity of upper bounding the error term $||\zeta - \textup{proj}_{\widehat{\Psi}}\left( \zeta \right)||_{\infty}$. At this point, by deriving symmetric results (i.e., Lemma \ref{lemma:error-prop} and Lemma \ref{lemma:error-prop-high-prob}), it is possible to derive an identical upper-bound for $\inf_{{r} \in \mathcal{R}_{{\bar{\mathfrak{B}}}}}||r - \hat{r}||_\infty$ for any choice of $\hat{r} \in \mathcal{R}_{\widehat{\bar{\mathfrak{B}}}}$. As a consequence, Lemma \ref{lemma:error-prop-high-prob} provides the following high-probability upper-bound on the Hausdorff distance between $\mathcal{R}_{\widehat{\bar{\mathfrak{B}}}}$ and $\mathcal{R}_{{\bar{\mathfrak{B}}}}$:
\begin{align}\label{eq:hauss-hp}
H_{\infty} \left( \mathcal{R}_{{\bar{\mathfrak{B}}}}, \mathcal{R}_{\widehat{\bar{\mathfrak{B}}}} \right) \le \frac{2\sqrt{2}\gamma}{1-\gamma} \beta_t + \left( \rho_t + {\frac{2\sqrt{2}\gamma}{1-\gamma} \left(\alpha_t + \beta_t \right)} \right) \min\left\{ \pi_{\textup{min}}^{-1} \max_i \xi_i, (1-\gamma)^{-1} \right\}.
\end{align}

\upperbound*
\begin{proof}
Let us start from Lemma \ref{lemma:error-prop-high-prob} and consider $t$ such that the condition of Lemma \ref{lemma:error-prop-high-prob} are satisfied. As previously discussed, due to Lemma \ref{lemma:error-prop-high-prob}, we have that, at time $t$, the algorithm US-IRL-SE induces an estimated feasible reward set such that the following holds with high-probability:
\begin{align*}
H_{\infty} \left( \mathcal{R}_{{\bar{\mathfrak{B}}}}, \mathcal{R}_{\widehat{\bar{\mathfrak{B}}}} \right) \le \frac{2\sqrt{2}\gamma}{1-\gamma} \beta_t + \left( \rho_t + {\frac{2\sqrt{2}\gamma}{1-\gamma} \left(\alpha_t + \beta_t \right)} \right) \min\left\{ \pi_{\textup{min}}^{-1} \max_i \xi_i, (1-\gamma)^{-1} \right\}.
\end{align*}
To conclude the proof, we need to find a sufficiently large $t$ such that the following holds:
\begin{align*}
H_{\infty} \left( \mathcal{R}_{{\bar{\mathfrak{B}}}}, \mathcal{R}_{\widehat{\bar{\mathfrak{B}}}} \right) \le \frac{2\sqrt{2}\gamma}{1-\gamma} \beta_t + \left( \rho_t + {\frac{2\sqrt{2}\gamma}{1-\gamma} \left(\alpha_t + \beta_t \right)} \right) \min\left\{ \pi_{\textup{min}}^{-1} \max_i \xi_i, (1-\gamma)^{-1} \right\} \le \epsilon,
\end{align*}
for any $\epsilon \in (0, 1)$. To this end, it is sufficient to find the smallest $t$ that satisfies the following conditions:
\begin{align}
& \rho_t  \min\left\{ \pi_{\textup{min}}^{-1} \max_i \xi_i, (1-\gamma)^{-1} \right\} \le \frac{\epsilon}{4} \label{eq:t-cond-1}\\
& \frac{2\sqrt{2}\gamma}{(1-\gamma)} \beta_t \le \frac{\epsilon}{4}\label{eq:t-cond-4} \\ 
& \frac{2\sqrt{2}\gamma}{(1-\gamma)} \min\left\{ \pi_{\textup{min}}^{-1} \max_i \xi_i, (1-\gamma)^{-1} \right\} \beta_t \le \frac{\epsilon}{4}\label{eq:t-cond-3} \\ 
& \frac{2\sqrt{2}\gamma}{(1-\gamma)} \min\left\{ \pi_{\textup{min}}^{-1} \max_i \xi_i, (1-\gamma)^{-1} \right\} \alpha_t \le \frac{\epsilon}{4} \label{eq:t-cond-2}
\end{align}

At this point, let us focus on $\rho_t  \min\left\{ \pi_{\textup{min}}^{-1} \max_i \xi_i, (1-\gamma)^{-1} \right\} \le \frac{\epsilon}{4}$. By simple algebraic manipulations we obtain the following sufficient condition for $t$:
\begin{align*}
t A \ge \frac{48 q_1^2 \log \left( \frac{3SAn}{\delta} \right)}{\pi_{\textup{min}} \epsilon^2}.
\end{align*}
For $\frac{2\sqrt{2}\gamma}{(1-\gamma)} \beta_t \le \frac{\epsilon}{4}$, instead, we obtain the following condition on $t$:\footnote{This follows by applying Lemma 12 in \citet{jonsson2020planning}}
\begin{align*}
t \ge \frac{128\gamma^2}{(1-\gamma)^2 \epsilon^2} \left[ \log\left( \frac{3SAn}{\delta} \right)  + (S-1) \log\left( \frac{16384 \gamma^4}{(1-\gamma)^4 \epsilon^4} \left(  \log \left( \frac{3SAn}{\delta} \right) + (S-1) \left( \sqrt{e} + \sqrt{\frac{1}{S-1}} \right)^2 \right) \right) \right].
\end{align*}
Similarly, for $\frac{2\sqrt{2}\gamma}{(1-\gamma)} \min\left\{ \pi_{\textup{min}}^{-1} \max_i \xi_i, (1-\gamma)^{-1} \right\} \beta_t \le \frac{\epsilon}{4}$, we need at least the following number of samples:
\begin{align*}
t \ge \frac{128 q_1^2 \gamma^2}{(1-\gamma)^2 \epsilon^2} \left[ \log\left( \frac{3SAn}{\delta} \right)  + (S-1) \log\left( \frac{16384 q_1^4 \gamma^4}{(1-\gamma)^4 \epsilon^4} \left(  \log \left( \frac{3SAn}{\delta} \right) + (S-1) \left( \sqrt{e} + \sqrt{\frac{1}{S-1}} \right)^2 \right) \right) \right]
\end{align*}
Finally, for $\frac{2\sqrt{2}\gamma}{(1-\gamma)} \min\left\{ \pi_{\textup{min}}^{-1} \max_i \xi_i, (1-\gamma)^{-1} \right\} \alpha_t \le \frac{\epsilon}{4}$, we obtain:
\begin{align*}
tA \ge \frac{128 q_1^2 \gamma^2}{(1-\gamma)^2 \epsilon^2} \left[ \log\left( \frac{3SAn}{\delta} \right)  + (A-1) \log\left( \frac{16384 q_1^2 \gamma^4}{(1-\gamma)^4 \epsilon^4} \left(  \log \left( \frac{3SAn}{\delta} \right) + (A-1) \left( \sqrt{e} + \sqrt{\frac{1}{A-1}} \right)^2 \right) \right) \right].
\end{align*}

At this point, we notice that the total number of samples gathered from the generative at step $t$ is given by $tSA$.
Therefore, from the previous equations, together with the conditions of Lemma \ref{lemma:error-prop-high-prob}, we obtian the following sufficient condition to guarantee that US-IRL-SE is $(\epsilon, \delta)$-correct.
\begin{align}
tSA = \mathcal{O} \left( \max \left\{ \frac{q_1^2 S \log\left( \frac{1}{\delta} \right)}{\pi_{\textup{min}}^2 \epsilon^2}, \frac{q_2^2 SA \left(S + \log\left( \frac{1}{\delta} \right) \right)}{(1-\gamma)^2 \epsilon^2}, \frac{ S\log\left( \frac{1}{\delta} \right)}{\pi_{\textup{min}}^2} \right\} \right),
\end{align}
where the first term is due to Equation \eqref{eq:t-cond-1}, the second one comes from Equations \eqref{eq:t-cond-4}-\eqref{eq:t-cond-3}, and the last term arises from Equation \eqref{eq:time-pi}, and is independent w.r.t. the desired accuracy $\epsilon$.\footnote{All the other equations introduce negligible terms in the $\mathcal{O}(\cdot)$ notation.} Specifically, we notice that for sufficiently small values of $\epsilon$ (i.e., $\epsilon < q_1$), the last term is always dominated by the first one, which concludes the proof.

\end{proof}

\section{Stochastic Optimal Expert}\label{app:stochastic-optimal-expert}
In this section, we discuss how to extend our analysis to the case in which the optimal expert plays a stochastic policy. First of all, let us define $\pi_{\textup{min}, E_1}$ as the minimum probability with which the optimal expert plays its actions.\footnote{Notice that, in principle, $\pi_{\textup{min}, E_1} \ne \pi_{\textup{min}}$ that we defined for the sub-optimal experts.} 

At this point, we notice that, for obtaining sample-complexity guarantees of the US-IRL-SE algorithm, Lemma \ref{lemma:error-prop} implies that we are interested in learning \emph{which are} the actions that are played with positive probability by the optimal expert. In other words, for any non-zero vector $x \in \mathbb{R}^{S \times A}$, we want the following to hold with high-probability:
\begin{align}\label{eq:stoc-opt-expt}
\Big| \left( \bar{B}^{\pi_{E_1}} - \bar{B}^{\hat{\pi}_{E_1}} \right) x \Big| = 0.
\end{align}
To this end, one can apply Lemma D.3 of \citet{metelli2023towards}, which implies that with a number of samples that is:
\begin{align}\label{eq:stoc-opt-expt-samples}
\mathcal{O} \left( S \frac{\log\left( \frac{1}{\delta} \right)}{\log\left( 1 / (1-\pi_{\textup{min},E_1}) \right)} \right),
\end{align}
Equation \eqref{eq:stoc-opt-expt} holds w.p. at least $1-\delta$. Once this condition is verified, the rest of the proof of the complexity of US-IRL-SE follows identically to the one of Theorem \ref{theo:ub}.\footnote{Notice, indeed, that we can apply Lemma \ref{lemma:error-prop-high-prob} as-is, but introducing the further constraints that $t$ should be sufficiently large so that Equation \eqref{eq:stoc-opt-expt} holds.} We notice that this introduces an additional maximum in the result of the sample-complexity. However, it is also possible to extend the proof of the lower bound of \citet{metelli2023towards} (see Theorem D.1 in \citet{metelli2023towards}), to show that the dependency resulting from \eqref{eq:stoc-opt-expt-samples} is unavoidable. In this sense, US-IRL-SE remains minimax optimal whenever the performance of the sub-optimal experts are sufficiently close to the ones of the optimal agent. 

\section{Per-expert Complexity of IRL-SE}\label{app:new-learning-formalism}
In this section, we provide a description of how to change the PAC learning formalism to show that Equation \eqref{eq:lb-two} actually represents a lower bound to the number of data that is necessary to gather from each of the different sub-optimal experts.

Specifically, we define a learning algorithm for an IRL problem $\bar{\mathfrak{B}}$ as a tuple $\mathfrak{A} = \left(\tau, \nu \right)$, where $\tau$ is a stopping time that controls the end of the data acquisition phase, and $\nu  = \left( \nu_t \right)_{t \in \mathbb{N}}$ is a history-dependent sampling strategy over $\mathcal{S} \times \mathcal{A} \times \left(\mathcal{S}\right)^{n+1}$. More precisely, $\nu_t \in \Delta^{\mathcal{S} \times \mathcal{A} \times \left(\mathcal{S}\right)^{n+1}}_{\mathcal{D}_t}$, where $\mathcal{D}_t = \left(\left(\mathcal{S} \times \mathcal{A}\right)^{n+2} \times \mathcal{S} \right)^{t}$ . At each time step $t \in \mathbb{N}$, the algorithm selects, by means of $\nu_t$:
\begin{itemize}
\item[(i)] a state-action pair $(S_t, A_t)$ and observes a sample $S'_t \sim p(\cdot | S_t, A_t)$ from the environment
\item[(ii)] a state $S_t^{(i)}$ for each expert $i \in \left\{1, \dots, n+1 \right\}$ and observes a sample $A_t^{(i)} \sim \pi_{E_i}(\cdot|S_t^{(i)})$
\end{itemize}
The observed realizations are then used to update the sampling strategy $\nu_t$, and the process goes on until the stopping rule is satisfied. At the end of the data acquisition phase, the algorithm leverages the collected data to output the estimate of the feasible reward set ${\mathcal{R}}_{\widehat{\bar{\mathfrak{B}}}_\tau}$ that is induced by the resulting empirical IRL problem $\widehat{\bar{\mathfrak{B}}}_\tau$. Given this formalism, we are interested in designing learning algorithms that, for any desired accuracy $\epsilon \in (0, 1)$ and any risk parameter $\delta \in (0, 1)$, guarantee that:
\begin{align}\label{pac:def-new}
\mathbb{P}_{\mathfrak{A},\bar{\mathfrak{B}}} \left( H_{\infty}(\mathcal{R}_{\bar{\mathfrak{B}}}, {\mathcal{R}}_{\widehat{\bar{\mathfrak{B}}}_\tau}) > \epsilon \right) \le \delta.
\end{align}
We refer to these algorithms as $(\epsilon,\delta)$-correct identification strategies. For $(\epsilon,\delta)$-correct strategies, we define their sample complexity as the sum of the number of \emph{unitary} samples gathered from the generative model. In other words, let us denote with $N_t(s,a)$ the number of samples gathered, at step $t$, from the environment, and let $N_t^{(i)}(s)$ denotes the number of samples gathered at step $t$ from the $i$-th expert at state $s$. Then, the sample complexity is given by:
\begin{align*}
\sum_{(s,a)} N_{\tau}(s,a) + \sum_{i} \sum_{s} N^{(i)}_{\tau}(s).
\end{align*}
Given this learning formalism, it is straightforward to extend the results of Theorem \ref{theo:lb} to this setting. More specifically, Theorem \ref{theo:lb-opt-exp} can be used to lower bound $\mathbb{E}\left[ \sum_{s} N^{(i)}_{\tau}(s) \right]$ for each sub-optimal expert, thus showing a significant increase in the sample complexity (i.e., linear in the number of sub-optimal experts).

\section{Further details on US-IRL-SE}\label{app:details-us-irl-se}

In this section, we provide further details on US-IRL-SE. More precisely, we specify the exact expression of the parameter $m$ that is used to provide $(\epsilon, \delta)$-correc

\paragraph{Expression of $m$} From the proof of Theorem \ref{theo:ub}, it is possible to derive an exact expression of $m$ that can be used in US-IRL-SE. More specifically, since $t = N_t(s,a)$,\footnote{This is a direct consequence of the uniform sampling strategy.} it is sufficient to take $m$ as the minimum $t$ that satisfies Equations \eqref{eq:t-cond-1}-\eqref{eq:t-cond-2} and Equations \eqref{eq:time-pi}-\eqref{eq:time-A}.


\paragraph{Computational Complexity of US-IRL-SE} At each iteration, US-IRL-SE will query the generative model in all state-action pair $(s,a) \in \mathcal{S} \times \mathcal{A}$. Notice that the generative model samples (i) the environment (ii) all the expert's policy. Assuming a unitary cost for generating each sample, a single query to the generative model results in a computational complexity of $\mathcal{O}\left( n \right)$. Therefore, the total computational complexity of US-IRL-SE is given $\mathcal{O} \left( SAmn \right)$.

\end{document}